\documentclass{article}

\usepackage{microtype}
\usepackage{graphicx}
\usepackage{booktabs}
\usepackage{xcolor}
\usepackage{hyperref}
\usepackage{amsmath}
\usepackage{amssymb}

\newtheorem{proposition}{Proposition}

\newtheorem{corollary}{Corollary}

\usepackage[accepted]{icml2026}

\icmltitlerunning{Position: Prioritize Identifying Structure, Not Complex Models, for Scientific Discovery}

\begin{document}

\twocolumn[
  \icmltitle{Position: Prioritize Identifying Structure, Not Complex Models, for Scientific Discovery}

  \begin{icmlauthorlist}
    \icmlauthor{Tyler~H.~McCormick}{uw}
  \end{icmlauthorlist}
  \icmlaffiliation{uw}{Departments of Statistics and Sociology, University of Washington, Seattle, WA, USA}
  \icmlcorrespondingauthor{T.~H.~McCormick}{tylermc@uw.edu}

  \icmlkeywords{Identifiability, Machine Learning, Miasma, Scientific Discovery}

  \vskip 0.3in
]

\printAffiliationsAndNotice{}

\begin{abstract}
Modern Machine Learning (ML) and Artificial Intelligence (AI) models, especially large language models (LLMs), are increasingly used to generate scientific hypotheses and mechanistic explanations from observational data. This position paper argues that in the high-dimensional proxy regimes where modern ML excels, mechanistic learning is generically underdetermined: many incompatible mechanisms induce essentially the same observational relationships on the support of the data, so predictive success and coherent explanations are insufficient evidence of mechanism discovery. This underdetermination becomes uniquely hazardous with large language models (LLMs), which tend to collapse large equivalence classes of explanations into a single fluent narrative. This paper proposes concrete standards for ``mechanistic ML,'' and argues these norms are necessary if LLM-centered workflows are to support science rather than merely simulate it.
\end{abstract}

\section{Introduction}

\textbf{Position:} \textbf{Absent explicit \emph{identifying structure}, high-dimensional ``scientific discovery'' from observed features is generically underdetermined; predictive success and fluent explanations are insufficient evidence of mechanism discovery.} \textbf{Therefore, research should prioritize articulating and evaluating identifying assumptions and discriminating regimes, rather than building ever more complex models.}
There is a substantial and growing body of work that frames complex AI and ML models as tools for scientific discovery, including symbolic regression \citep{Schmidt2009,Udrescu2020}, sparse identification of dynamics and PDEs \citep{Brunton2016,Rudy2017}, physics-informed inverse problems \citep{Raissi2019}, and methods that extract interpretable hypotheses from high-dimensional text via sparse autoencoders and LLMs \citep{pmlr-v267-movva25a}. In the social sciences, the Fragile Families Challenge tested the value of prediction systems for social outcomes \citep{Salganik2020}, and ``agnostic'' approaches argue for more inductive, sequential discovery workflows \citep{Grimmer2021}. There is great promise in incorporating AI/ML tools in scientific discovery. 

This position paper, however, contends that using modern AI/ML tools agnostically to uncover mechanisms is inherently fraught.  Put simply, there are many representations of the relationship between observed features and outcomes that are consistent with any observed dataset. A representation can yield strong in-domain prediction and even convincing simulations, yet it provides no guarantee that the learned representation corresponds to the underlying, intervention-stable driver of the system.  The remedy is to make identifying structure explicit.  By \emph{identifying structure}, the paper means joint specification of (a) restrictions on the latent mechanism (what the system is allowed to do under interventions) and (b) restrictions on the observation process (how latent drivers generate measured proxies), together with (c) a data-collection regime that supplies discriminating variation (new environments, interventions, or measurement channels).
Identifying structure and complex models are not opposites. In many successful scientific-ML systems, complex models are exactly what make it possible to exploit identifying structure once it exists. %
This point echoes Cartwright's \emph{nomological machine}: law-like regularities arise when components with stable capacities are arranged in a sufficiently stable environment so that repeated operation yields repeatable behavior \citep{cartwright1999nomological}. %

The argument starts from a key observation: \emph{scientific workflows rarely observe the drivers of scientific processes directly.} Instead, they observe \emph{proxies} that are associated with one or more latent drivers. In the apocryphal story, Newton didn't see gravity; he felt a bump on his head from the falling apple.  \citet{mendel1866} observed phenotypes, traits such as seed shape and seed color, recorded for parent plants and their offspring.  These phenotypes are proxies that, when observed across a set of controlled crosses, gave clues about the underlying scientific mechanism, the transmission of discrete hereditary ``factors'' across generations.  Crucially, Mendel's success did not come from black-box prediction on high-dimensional proxies, but from \emph{identifying structure}: favorable biological architecture, carefully chosen discrete traits, controlled matings, replication, and a narrow hypothesis class (segregation and independent assortment) against which the proxy patterns could be tested and falsified \citep{mendel1866,fisher1936,curtis2023mendel}. For the purpose of this paper, a \emph{mechanistic query} is a concrete question about an underlying data-generating process, such as what would happen under an intervention, whether a relationship is stable across settings, which latent driver explains an observed pattern, or what sign a causal effect has. Such a query \(q(M)\) is identified when every mechanism compatible with the observed proxy evidence and stated assumptions gives the same answer, and partially identified when the compatible mechanisms give a bounded or otherwise structured set of possible answers. %

The fundamental challenge is that mechanistic inquiry needs to understand both \emph{how good the proxy is} (i.e., how strong the relationship is between the proxy and driving mechanism) and \emph{what it is a proxy for} (i.e., the underlying scientific mechanism).  The position here is that, in high-dimensional observational regimes, both cannot be identified without additional structure. The shift from bureaucratic statistical regimes toward brokerage regimes of passively collected digital traces makes this proxy problem increasingly central in social and administrative data \citep{fourcade2026census}. %
Unstructured discovery, at best, identifies an \emph{equivalence class} of mechanisms based on plausible relationships between the proxies and the mechanisms. Formally, the \emph{proxy observational law} \(P_X\) is the population distribution of the observed proxies under the observational regime at hand. For a fixed proxy observational law \(P_X\), let \(\mathcal M(P_X)\) denote the set of
mechanisms and measurement channels that induce \(P_X\). Without identifying structure, \(\mathcal M(P_X)\)
is typically not a singleton. \citet{ogburn2021warning} gives examples in the context of causal graphs.

Improving model performance or even the existence of a plausible, consistent explanatory channel is also not sufficient.  Coherence is not, by itself, scientific evidence. What's more, fixating on the wrong mechanism can have dramatic, longstanding consequences.  Take, for example, the miasma, or ``bad air'' theory of disease.  The proxies are environmental correlates of illness, such as odor, proximity to swamps/sewers, stagnant water, signs of decaying material, or meteorological conditions.  Many of these proxies are strongly correlated with outbreaks, providing seemingly ironclad evidence for the theory.  The strength of these associations, along with the tendency for path dependence in science~\citep{kuhn1997structure,soler2025would}, meant that the miasma theory remained pervasive even in the presence of mounting evidence for germ-based disease transmission. Mendel, on the other hand, succeeded because favorable biology, selected measurements, and discriminating design effectively made \(\mathcal M(P_X)\)
small for the mechanistic queries he asked.  Throughout, the term \emph{design} is used to refer to the structure of the evidence-generating regime, including which variables are measured, which labels or environments are recorded, which naturally occurring contrasts are available, and which interventions or controlled protocols, if any, produced the observations. %

A richer feature space might seem capable of identifying underlying mechanisms from proxies, but high-dimensional proxies often concentrate data on a thin effective support shaped by institutions, selection, and measurement. As the proxy dimension increases, the number of possible patterns that, spuriously or otherwise, give a representation that's consistent with the data also increases.  The inevitable result is that many incompatible mechanisms can agree closely on what they support yet diverge sharply under interventions or modest domain shifts. Formally, Appendix~\ref{app:stronger} shows that if $P_X$ is supported on a lower-dimensional set,
then off-support behavior is not identified for any rich function class without additional constraints. LLMs add a distinctive hazard. Recent work formalizes an ``inevitability'' result for open-world querying: for any fixed computable \emph{learner}, there exist computable target functions on which it must systematically err on infinitely
many inputs \citep{Xu2025}. This paper argues that an analogous inevitability holds for \emph{explanations}: for any observational evidence based on proxies, there exist multiple compatible but conflicting mechanistic stories.
What is new about LLM workflows is not underidentification itself, as identification limits are classical, but that LLMs \emph{industrialize the collapse} of underidentification into a single fluent narrative at scale.

This contention does not depend on a narrow definition of mechanism. The term means stable structural features that explain how a system responds to interventions: derivative constraints (which variables matter and with what sign), conservation laws, invariances, and causal pathways supporting policy counterfactuals \citep{Machamer2000,Woodward2003,Heckman2000}. Economists distinguish reduced-form associations, which may be accurate yet mechanistically agnostic \citep{Angrist2010,AtheyImbens2019}, from structural models encoding policy-relevant counterfactuals \citep{Heckman2000}.

The call to action is simple: for ML systems to contribute to scientific discovery, ``mechanistic learning'' must be treated as an identification problem, with research prioritizing the discovery and declaration of the identifying structure that makes mechanistic questions answerable from proxy data. %
Accordingly, any mechanistic claim must be paired with at least one of: (i) a clear statement of this identifying structure (and what remains unidentified), (ii) mechanism-discriminating evaluations that can shrink the proxy-compatible set (interventions, cross-environment invariance tests, derivative/shape constraints), or (iii) multiplicity reporting that characterizes the surviving equivalence class, including explicit falsifiers and sensitivity to the stated assumptions.

The paper proceeds as follows. Section~\ref{sec:setup} sets up mechanisms, proxies, and models; Section~\ref{sec:proxy_gap} shows that absent identifying structure, proxies identify mechanism equivalence classes rather than unique mechanisms; Section~\ref{sec:cod} explains how high dimensionality and thin support aggravate those classes; Section~\ref{sec:narrative} introduces LLM narrative collapse; Section~\ref{sec:uq} relates the ambiguity to aleatoric, epistemic, and structural identification uncertainty; Section~\ref{sec:example} gives the empirical example; Section~\ref{sec:cta} states the call to action; and Section~\ref{sec:alt} provides alternative views.

\section{Problem Setup and Terminology}\label{sec:setup}

Let \(S\) denote latent scientific drivers (``mechanism-level'' variables). The observed data consist of proxies \(X\) generated from \(S\) by a measurement process. Consider deterministic measurement \(X=f(S)\) and stochastic measurement channels \(X\mid S\sim K(\cdot\mid S)\),
where \(K\) is a Markov kernel from the latent space to the proxy space, covering additive noise, discretization,
aggregation, censoring, and other coarsenings \citep{fuller1987,carroll2006}.
Given a latent law \(P_S\) and measurement channel \(K\), the induced \emph{proxy observational law} is \(P_X(A)=\int K(A\mid s)\,P_S(ds)\) for measurable proxy events \(A\). The \emph{proxy gap} is the distinction between this observable law and the latent mechanism \(M=(P_S,K)\), or a mechanistic query \(q(M)\), when distinct mechanisms induce the same \(P_X\).
A modern workflow trains a predictor \(h\), often through a learned representation \(Z=\phi(X)\), to optimize an in-domain objective, such as predicting an outcome \(Y\) from proxies \(X\), which is distinct from recovering the latent \(S\).%

To move from proxies to mechanisms, many methods pool ``similar'' observations, either explicitly (clustering, trees, mixtures, nearest neighbors) or implicitly (representation learning that induces neighborhood smoothing). The workhorse assumption behind pooling is a symmetry choice often framed as \emph{exchangeability}: declaring certain differences irrelevant so that averaging is meaningful. Exchangeability yields powerful representation theorems \citep{definetti1937,aldous1981,kallenberg2005,diaconis1980}. But representation is not identification: the theorems specify what must be true \emph{if} a symmetry holds, but do not guarantee the symmetry is correct, nor that a unique mechanism is pinned down by proxies. Kleinberg’s clustering impossibility theorem makes this concrete \citep{kleinberg2002}. This paper makes three claims.  Proofs of the three forthcoming Propositions are in Appendix~\ref{app:proofs}. For clarity, the objects relate as:
\[
S \xrightarrow[\text{measurement}]{K \text{ or } f} X \xrightarrow[\text{encoding}]{\phi} Z \xrightarrow[\text{predictor}]{h} \widehat Y,
\]
Many predictors or encodings can map the same proxies to similarly good predictions, but that happens to the right of \(X\). The proxy gap lives to the left of \(X\): many latent mechanisms and measurement channels can generate the same proxies while implying different mechanistic answers. Identification concerns what can be inferred about \(S\) or \(q(M)\) from \(X\), not merely how well \(h\) predicts \(Y\) in-domain.

\section{Claim I: The Proxy Gap Creates Mechanism Equivalence Classes}\label{sec:proxy_gap}

Proxies constrain mechanisms only through a many-to-one measurement process (deterministic or stochastic). Without additional identifying structure, the evidence generally identifies an \emph{equivalence class} of mechanisms rather than a unique mechanism. Causal-quartet examples in \citet{mcgowan2024causal} show this in the context of causal inference.
For 19th-century epidemic disease, the day-to-day evidence available to observers---symptoms, odors, local air conditions, seasonality, crowding, sanitation, neighborhood mortality rates, and water/waste correlates---was often compatible with both miasmatic and germ-theoretic accounts. Miasmatic theories were empirically plausible because foul environments were genuinely predictive of disease, and sanitation often improved health; germ-theoretic accounts became compelling only when new evidence decoupled these proxies from specific microbial exposure and transmission mechanisms. With only the original proxy channel and no new measurement or identification restriction, observing more cases need not resolve the mechanistic ambiguity.

The argument begins with an intentionally ``easy'' setting.
Assume the latent state decomposes into \(B\) blocks \(S=(S_1,\dots,S_B)\).
A \emph{proxy block} \(X_b\) is the collection of observed measurements intended to capture (possibly noisily)
the corresponding latent block \(S_b\). Assume a modular measurement structure in which each
proxy block depends only on its corresponding latent block (e.g., multiple sensors targeting the same driver),
either deterministically \(X_b=f_b(S_b)\) or via a channel \(X_b\mid S_b\sim K_b(\cdot\mid S_b)\).
The propositions below show that even if this modular structure holds and measurement is invertible within
blocks, proxies alone still do not uniquely identify the latent mechanism without additional identifying
assumptions or discriminating regimes.
\begin{proposition}[Proxy identifiability]\label{thm:proxy_nogo}
Let $X=(X_1,\dots,X_B)$ be observed proxy blocks. Suppose an admissible deterministic representation consists of latent blocks $S=(S_1,\dots,S_B)$ and bimeasurable bijections $f_b:\mathcal S_b\to\mathcal X_b^{\mathrm{supp}}$ such that $X_b=f_b(S_b)$ a.s. for each $b=1,\dots,B$. If a second admissible representation of the same proxy blocks is given by $X_b=\tilde f_{\pi(b)}(\tilde S_{\pi(b)})$ a.s. for each $b$, where $\pi$ is a permutation of $\{1,\dots,B\}$ and each $\tilde f_{\pi(b)}$ is also a bimeasurable bijection onto $\mathcal X_b^{\mathrm{supp}}$, then there exist bimeasurable bijections satisfying, for each $b$,
\[
\begin{gathered}
g_b:=\tilde f_{\pi(b)}^{-1}\circ f_b,\qquad
\tilde S_{\pi(b)}=g_b(S_b)\ \text{a.s.},\\[-0.25ex]
\tilde f_{\pi(b)}=f_b\circ g_b^{-1}.
\end{gathered}
\]
\end{proposition}
In this deterministic invertible setting, the equivalence class induced by $P_X$ is exactly within-block reparameterizations plus, when labels are not intrinsic, block permutations. Thus only quantities invariant under those transformations (e.g., certain conditional independences or invariance relations) are identified from $P_X$ without further structure. This population-level ambiguity is not a finite-sample issue and does not disappear as parameter uncertainty shrinks; Proposition~\ref{thm:proxy_nogo_channel} shows it also survives arbitrary conditionally independent measurement noise.

\begin{samepage}
\begin{proposition}[Non-identification persists under noise]\label{thm:proxy_nogo_channel}
Let $S=(S_1,\dots,S_B)$ and let $K(x\mid s)=\prod_{b=1}^B K_b(x_b\mid s_b)$ be a conditionally independent proxy channel. For any blockwise bimeasurable maps $g_b$, define
\[
\tilde S_b=g_b(S_b),\qquad
\tilde K_b(\cdot\mid \tilde s_b):=K_b(\cdot\mid g_b^{-1}(\tilde s_b)).
\]
Then $(S,K)$ and $(\tilde S,\tilde K)$ induce the same distribution of $X$. Thus, blockwise reparameterization
non-identification survives arbitrary conditionally independent measurement noise.
\end{proposition}
\end{samepage}
In short: richer proxies can improve \emph{estimation} of a chosen model, but they do not, by themselves, guarantee \emph{identification} of a mechanism.

\section{Claim II: High Dimensionality Aggravates the Equivalence-Class Problem}\label{sec:cod}

The proxy gap already creates mechanism equivalence classes. High dimensionality and modern ML practice aggravate this problem by making disagreement easier to hide on thin support and harder to diagnose in realistic workflows.  High-dimensional proxy vectors rarely fill their ambient space; they concentrate on a thin effective support shaped by selection, measurement, and correlation structure. At the same time, the curse of dimensionality implies that learning general functions (not to mention mechanism-relevant counterfactual behavior) requires sample sizes that explode with dimension unless strong structure is assumed~\cite{Bellman1957DynamicProgramming,Stone1982OptimalRates}. %

Two geometric consequences are especially relevant. First, in high dimensions, distances can concentrate: ``nearest'' and ``farthest'' neighbors can become nearly indistinguishable under broad conditions~\citep{beyer1999,aggarwal2001}. When neighborhoods lose meaning, any method that depends on local pooling (nearest neighbors, kernels, clustering, or representation-induced neighborhood smoothing) becomes sensitive to modeling choices that are not themselves identified by the data. Second, even coarse partitioning of a $d$-dimensional proxy space yields exponentially many regions as $d$ grows. Unless sample sizes grow commensurately, most regions contain too little data to discriminate mechanisms that agree on the observed support but differ elsewhere. Operationally, this is how a large mechanism equivalence class shows up in finite samples: many incompatible mechanisms are observationally indistinguishable on the thin support where data live, yet can diverge sharply off-support.

In sum, Sections~\ref{sec:proxy_gap}--\ref{sec:cod} argue that in proxy-rich, high-dimensional regimes, mechanistic claims should be treated as claims about a \emph{set} of observationally compatible mechanisms unless accompanied by identifying structure or mechanism-discriminating evaluation.

\section{Claim III: An Additional LLM-Specific Hazard---Narrative Collapse as False Resolution}\label{sec:narrative}

Sections~\ref{sec:proxy_gap}--\ref{sec:cod} argue that, without identifying structure, proxy evidence often supports a set of observationally compatible mechanisms rather than a unique mechanism. LLM-centered workflows create a distinctive hazard when they represent that set as a resolved explanatory account. This paper calls this \emph{narrative collapse}: a many-to-one mapping from a mechanism compatibility set to an explanatory output---a single story, a ranked list of factors, or a polished synthesis---that is naturally read as more identified than the evidence warrants. Humans also do this---scientific communities routinely converge on ``the'' story under ambiguity, but LLMs industrialize the process by making single-story explanations fast and easy to produce at scale. 
Narrative collapse is not the claim that every LLM response names exactly one cause. Many systems, especially when prompted, enumerate several factors or caveats. A factor list avoids collapse only if it preserves the relevant scientific multiplicity: which mechanisms are mutually incompatible, which assumptions make each admissible, which answers to \(q(M)\) are stable across \(\mathcal M(P_X)\), and what evidence would distinguish them. %

This use of ``collapse'' is intentionally different from two nearby concerns. \emph{Model collapse} refers to degenerative dynamics that can arise when generative models are trained recursively on model-generated content, leading to distributional narrowing and loss of tails~\citep{shumailov2024modelcollapse}. \emph{Algorithmic monoculture} refers to population-level harms when many decision-makers converge on the same algorithm, reducing system robustness even when the algorithm is locally better for each agent~\citep{kleinberg2021monoculture}. Both phenomena can amplify narrative collapse, but neither is required for it: narrative collapse can occur in a single interaction with a single model on a fixed dataset, whenever the evidence does not identify a unique mechanism and the interface reports a resolved explanation without representing what remains compatible.

How often this occurs in present systems is an empirical question, and the rate will vary by model, prompt, interface, task, and user population. The structural claim is conditional: whenever a workflow asks for, or rewards, a resolved mechanistic answer while \(\mathcal M(P_X)\) contains mechanisms with different values of \(q(M)\), no such answer can be uniformly warranted. Better interfaces can reduce this risk by surfacing the compatible mechanism set, tagging assumptions, and proposing discriminating tests. That is precisely the design recommendation rather than an exception to the argument. Recent critiques of uncertainty quantification for LLM agents emphasize that the standard aleatoric/epistemic dichotomy becomes strained in open, interactive settings~\citep{kirchhof2025uq}. The point is complementary: in proxy-driven science, the dominant ambiguity is often \emph{structural identification uncertainty}, and a resolved narrative is a lossy representation of that uncertainty.

To make the issue precise without over-formalizing, fix an observed proxy distribution \(P_X\) and let
\(D=(X^{(1)},\dots,X^{(n)}) \sim P_X^n\) denote the observed dataset. Let \(\mathcal M(P_X)\) denote the set of mechanisms
(latent data-generating processes for \(S\) together with measurement processes \(K\)) that are observationally compatible with \(P_X\) under the setup of Sections~\ref{sec:setup}--\ref{sec:proxy_gap}. Let \(q:\mathcal M \to \mathcal A\) be a mechanistic query mapping a mechanism to an answer space \(\mathcal A\) (e.g., ``Which variables are causally upstream?'' ``What intervention increases \(Y\)?'' ``Which invariance should hold out of domain?''). Everything below holds verbatim if \(D\) includes outcomes \(Y\) or other observables: simply replace \(P_X\) by the joint observational law of whatever variables the workflow conditions on. Proposition~\ref{prop:narrative} records the basic obstruction: if two proxy-compatible mechanisms imply different answers to \(q\), then no rule that sees only the proxies can always give the right single answer.

\begin{proposition}[Narrative collapse as minimax ambiguity]\label{prop:narrative}
If there exist \(M_1,M_2\in \mathcal M(P_X)\) such that \(q(M_1)\neq q(M_2)\), then no single-valued explanation rule
\(\widehat a=\widehat a(D)\) can be uniformly correct over \(\mathcal M(P_X)\) for the query \(q\).
In particular, under 0--1 loss \(\ell(\widehat a,q(M))=\mathbf 1\{\widehat a\neq q(M)\}\), the worst-case risk satisfies
\[
\inf_{\widehat a}\ \sup_{M\in\mathcal M(P_X)} \ \mathbb E\big[\ell(\widehat a(D),q(M))\big] \ \ge \ \tfrac{1}{2},
\]
where the expectation is over \(D\sim P_X^n\) (equivalently, over any procedure that only has access to the proxies).
\end{proposition}

Proposition~\ref{prop:narrative} is deliberately illustrative: when the evidence admits multiple incompatible answers to a mechanistic query, any \emph{single} reported answer must be wrong for at least one observationally compatible mechanism. The scientifically responsible response is therefore not ``be more eloquent''; it is to change what is reported. %
In interactive use, the risk is amplified when follow-up questions move from description to ``why'' and ``what-if'' claims. Those questions can require extrapolating beyond what the proxies identify. Conversational norms can also reward \emph{consistency}: once a workflow commits to a mechanism in an early turn, later turns may rationalize and elaborate that commitment rather than reopening the mechanism set. This coherence pressure is useful for ordinary assistance, but under underidentification it can turn a set-valued scientific state of knowledge into an increasingly entrenched point narrative.
Narrative collapse is therefore the mechanism-level analogue of overconfident extrapolation: it converts a set-valued scientific state of knowledge into a resolved explanatory object. Operationalizing the concept requires benchmarks that specify (i) proxy data, (ii) a mechanistic query \(q\), (iii) the set of observationally compatible answers under stated assumptions, and (iv) at least one discriminating test that would shrink that set. The unit of evaluation should not be whether a response names one cause or many factors, but whether it preserves the identified set: compatible-answer coverage, separation of mutually incompatible mechanisms, assumption tagging, discriminating-test quality, and identified-set calibration---whether the system makes point claims only when the benchmark enters an identified regime.

\section{Where This Sits Relative to Aleatoric and Epistemic Uncertainty}\label{sec:uq}
Uncertainty quantification in ML is often organized around a dichotomy: \emph{aleatoric} uncertainty, attributed to irreducible randomness in outcomes, and \emph{epistemic} uncertainty, attributed to an agent's lack of knowledge about parameters, models, or hypotheses \citep{kendall2017uncertainties,huellermeier2021aleatoricepistemic}. This split is useful for prediction problems, but it does not cleanly capture the central phenomenon in Sections~3--5. Even with unlimited data about proxies \(X\), mechanism-level claims about \(S\) can remain underdetermined unless one adds explicit identifying structure. This paper isolates a particularly important form of epistemic uncertainty: uncertainty induced by non-identification from proxies.

To make this precise (with a more formal presentation in Appendix~\ref{app:compare}), fix an observational distribution \(P_X\). Consider the set of mechanism stories that could have produced it:
$
\mathcal{M}(P_X) \;=\; \bigl\{ (P_S,\text{measurement map/channel}) \,:\, (S \to X)\ \text{induces}\ P_X \bigr\}.
$
Sections~3--4 argue that \(\mathcal{M}(P_X)\) is typically large in modern regimes: proxies are many-to-one summaries of latent drivers, and high dimensionality makes disagreement outside the observed support easier to hide and harder to diagnose. This is the \emph{structural identification uncertainty} emphasized here. In the context of causal inference,~\citet{gelman2024causalquartets} show that even when a scalar average causal effect is held fixed, very different heterogeneous effect patterns can remain compatible with that same summary, for example. In contrast, aleatoric uncertainty concerns variability of outcomes \emph{conditional on a mechanism}, while common uses of epistemic uncertainty quantify uncertainty relative to a chosen model or hypothesis class. The proxy gap matters because it questions whether that class and its measurement geometry identify the mechanistic query at all. 

This perspective clarifies how these claims relate to (and differ from) the Rashomon effect and underspecification. Rashomon-style multiplicity concerns many predictors achieving similar risk \citep{fisher2019allmodels,xin2022rashomon,venkateswaran2024robustly}; underspecification emphasizes that many models can fit training objectives yet behave differently under deployment \citep{damour2022}. Structural identification uncertainty is more basic: from \(X\) alone, the mapping back to mechanism \(S\) is not pinned down without additional assumptions on measurement and structure. Rashomon and underspecification are important amplifiers in finite samples and deployment, but the proxy gap can generate mechanism equivalence classes even before optimization and architecture enter the picture.

These distinctions matter for LLM-centered scientific workflows. Many uncertainty quantification (UQ) tools quantify dispersion around a well-defined predictive target. But when the user asks a mechanism question—``what explains this pattern?''—the dominant uncertainty may be which members of \(\mathcal{M}(P_X)\) remain plausible. In that regime, a single calibrated interval or scalar ``epistemic uncertainty'' can be actively misleading: it suggests that what is missing is merely more data or better parameter estimates, when the binding constraint is lack of identification structure. This issue is sharper in interactive settings, where queries expand into counterfactuals and the system is forced to pick a story off-support; recent critiques argue that the aleatoric/epistemic split becomes strained for LLM agents in open interaction \citep{kirchhof2025uq}. The point here is complementary: even perfect predictive calibration can leave mechanistic claims non-identified.

\section{An Example with Pea Plants}\label{sec:example}
Mendel crossed a pure-breeding wrinkled-green line with a pure-breeding round-yellow line. The first generation (F$_1$) was uniformly round-yellow, and those F$_1$ plants were crossed again to produce the F$_2$ generation analyzed here. The observed labels are visible traits---round-yellow (\texttt{RY}), round-green (\texttt{RG}), wrinkled-yellow (\texttt{WY}), and wrinkled-green (\texttt{WG})---not the hereditary state itself. The hidden mechanism is the rule by which parental hereditary factors produce those visible categories. The example is intentionally favorable: Mendel worked with unusually discrete, selected, experimentally tractable traits; many genetic traits involve many loci, interactions, pleiotropy, and environmental dependence \citep{curtis2023mendel,bapty2023mendel,mackay2024pleiotropy}. This section asks what happens when that structure is removed from the evidence regime.

The simulation is designed to make three points concrete. First, under a phenotype-only observational channel, distinct mechanisms can be observationally equivalent on the support of the data (Claim~I). Second, adding experimental structure (here: labeling the cross type and including additional crosses) shrinks this mechanism equivalence class (Claim~II). Third, modern ``black-box discoverers'' can achieve similar predictive performance while still yielding unstable mechanistic counterfactuals when the design metadata is removed (Claim~III).
The simulation compares two regimes. Regime~A gives only pooled F$_2$ phenotype counts: ``here are the offspring categories.'' Regime~B also records which controlled cross produced each observation (\texttt{F2\_self}, \texttt{testcross}, \texttt{monoA}, or \texttt{monoB}) and includes additional labeled crosses. This cross label is the identifying structure: phenotype counts alone record what the offspring looked like, but the mating protocol tells us whether those counts came from an F$_2$ self-cross, a testcross, or a single-locus cross, which are different tests of the inheritance rule.

The example compares two stories that look identical if only the pooled F$_2$ counts are inspected. The Mendelian story says those counts come from a real inheritance rule: two traits, dominance, and independent assortment, which produce the familiar \(9{:}3{:}3{:}1\) pattern. The deliberately simple alternative ignores that inheritance rule and instead treats shape and color as two separate weighted coin flips, each producing the dominant trait with probability \(3/4\). This alternative is constructed to match the pooled F$_2$ counts exactly. But once the data record which cross produced each offspring, the two stories no longer agree. In \emph{Regime A} (phenotype-only), the simulation generates \(n_{\text{F2}}=6000\) offspring from an F$_2$-self distribution (parents \texttt{RY} \(\times\) \texttt{RY}), and the analysis evaluates fit only at the level of the pooled phenotype distribution. In \emph{Regime B} (design), the simulation generates additional labeled crosses, each with \(n_{\text{per cross}}=1200\) offspring: \texttt{F2\_self}, \texttt{testcross}, \texttt{monoA}, and \texttt{monoB}. The key identifying structure is the cross-type label (e.g., testcross vs mono cross), which encodes experimental lineage/design information that is not recoverable from phenotype proxies alone but radically changes what the same proxy observations imply.
The analysis fits a two-hidden-layer Multilayer Perceptron (MLP) classifier with architecture \((32,32)\) to predict offspring phenotype from parent phenotypes. In the figure labels, ``no design'' is Regime~A: the MLP trains without cross labels and sees only F$_2$ self-cross data. ``With design'' is Regime~B: the MLP trains with the cross label and sees the queried testcross support. A seed is one independently trained MLP run with a different random-state value controlling random initialization and training randomness; the train/test split is fixed across seeds. The ``all'' bars use all 60 such runs in the relevant regime; the ``near-opt'' bars restrict to seeds whose held-out test log loss is within \(\varepsilon=0.01\) of the best seed for that same regime.

\begin{figure}[h!]
  \centering
  \includegraphics[width=.8\linewidth]{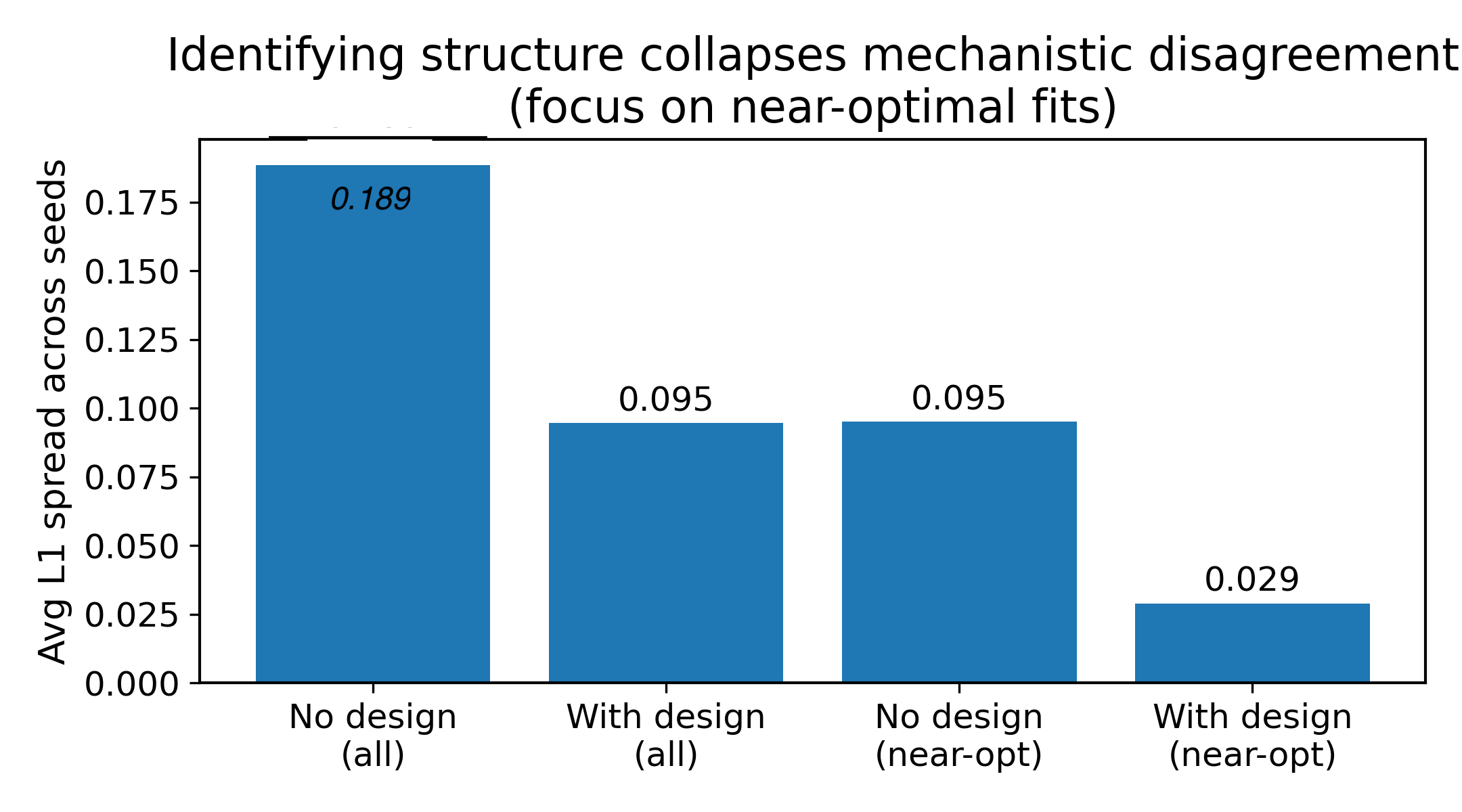}
  \caption{\textbf{Identifying structure collapses mechanistic disagreement.}
  Mean \(\ell_1\) spread of counterfactual predicted probability vectors across MLP runs. ``No design'' is Regime~A (phenotype-only F$_2$ self-cross data); ``with design'' is Regime~B (cross labels plus additional labeled crosses, including testcross support). ``All'' uses all 60 random-state runs; ``near-opt'' restricts to runs within \(\varepsilon=0.01\) of the best held-out test log loss in the same regime. The spread collapses in the design regime, which adds testcross support and labels that make the counterfactual cross prediction well defined.}
  \label{fig:peas_spread}
\end{figure}

The prediction task asks: for the same observed parent phenotype pair, \texttt{RY}\(\times\)\texttt{WG}, what offspring probabilities should the model assign to \texttt{RG}, \texttt{RY}, \texttt{WG}, and \texttt{WY}? Without cross labels, this question is underdetermined because the same parent phenotypes can arise from different experimental lineages; with labels, it becomes a specific testcross prediction (\texttt{cross=testcross}). The mean \(\ell_1\) spread summarizes how much the fitted models disagree:
\(\frac{1}{N_{\text{seeds}}}\sum_{r=1}^{N_{\text{seeds}}} \| \hat p_r - \bar p \|_1\), where \(\bar p := \frac{1}{N_{\text{seeds}}}\sum_{r=1}^{N_{\text{seeds}}}\hat p_r\),
computed either across all seeds or only near-optimal seeds. Zero would mean every fitted network gives the same four probabilities. Figure~\ref{fig:peas_spread} shows the point: in Regime~A, equally good predictors disagree more about the counterfactual cross; in Regime~B, the added cross labels and testcross support make the answer substantially more stable. The near-optimal spread falls from \(0.095\) without design to \(0.029\) with design, while the all-seed spread falls from \(0.189\) to \(0.095\). Additional concrete probabilities and plots are in Appendix~\ref{app:pea}; a continuous low-dimensional simulation is in Appendix~\ref{app:sim}.

\section{Call to Action: Mechanistic ML}\label{sec:cta}

The core prescription is to treat \emph{identification}, not expressiveness, as the binding constraint in proxy-rich scientific discovery. In this paper, that means identifying the answer to a stated mechanistic query from a compatibility set of mechanisms, not necessarily estimating a treatment effect from a randomized or quasi-random treatment assignment. First-principles restrictions are archetypal identifying structure: conservation laws, symmetries, monotonicity, geometry, and intervention semantics all shrink the compatible mechanism set and make mechanistic claims auditable. Identification is a common concept in economics and causal inference, among other areas.  For example, unsupervised disentanglement is provably non-identifiable without inductive biases and/or supervision \citep{pmlr-v97-locatello19a}, while identifiability can be recovered under additional observed variables, grouping structure, or restricted ICA-style assumptions \citep{pmlr-v108-khemakhem20a,pmlr-v235-morioka24a}. More broadly, the logic of no-free-lunch and Kolmogorov complexity arguments makes the same meta-point: without inductive bias aligned with the world, learning cannot select the right explanation class \citep{goldblum2024nfl}. In the presence of proxies, mechanistic ML should adopt the same norms in pursuit of scientific discovery. The central question is not ``how flexible is the learner?'' but ``under what assumptions and what data design is the mechanistic query $q(M)$ identified?'' Appendix~\ref{app:ident_structure_examples} gives six examples.%
Appendix~\ref{app:ident_structure_def} defines identifying structure with mathematical formalism, while Appendix~\ref{app:min_ident} gives an example of how this conceptualization can be used to identify minimal identification assumptions for a specific class of models.  
This is also why several successful scientific-ML systems should be viewed as positive cases rather than counterexamples. Equation-discovery and inverse-problem methods such as symbolic regression, SINDy, and PINNs succeed precisely when they combine expressive function classes with explicit structural restrictions on the mechanism, the observation process, or the experimental regime. Their success comes not from black-box flexibility alone, but from the identifying structure that makes mechanistic queries auditable.
In practice, identifying structure yields three auditable manifestations. First, a workflow can \emph{declare assumptions} that restrict both the mechanism and the measurement channel, so that readers can see which equivalence classes have been ruled in and ruled out. Second, it can \emph{collect and evaluate on discriminating data}: interventions, invariances across environments, and derivative constraints that are implied by the mechanism but not by proxy correlations alone. This is where design enters: the call is not just for \emph{more} data, but for \emph{new} data that breaks observational equivalence by changing environments, manipulating inputs, or measuring closer to the latent drivers. Without a notion of identifying structure, there is no guidance about \emph{how} to collect data.  Third, when assumptions and available data still leave ambiguity, a workflow can \emph{report multiplicity}: characterize the remaining identified set of mechanistic answers (or mechanisms) rather than collapsing it into a single narrative.

Mendel's case worked because favorable biology and discriminating design lined up. On the mechanism side, his hypothesis class was sharply constrained: particulate inheritance with segregation and independent assortment. On the observation side, he used selected discrete phenotypes whose relationship to latent hereditary factors was stable enough to generate sharp, falsifiable implications. These restrictions were paired with controlled crosses, replication, and targeted contrasts, such as F$_2$ ratios and testcrosses, that would look different under alternative mechanisms \citep{mendel1866,curtis2023mendel}. The lesson is not that design alone made genetics simple, but that mechanistic discovery requires an evidence regime in which relevant alternatives imply different observable patterns. Miasmatic theories of epidemic disease, by contrast, were supported by rich but indirect proxies: odor, dampness, crowding, poor drainage, sewage, low elevation, seasonality, and neighborhood mortality. These proxies were genuinely predictive, and sanitary reforms motivated by miasmatic reasoning often improved health, which made the theory empirically plausible. But the same proxy patterns were also compatible with germ-theoretic accounts, because waste, water, housing, and occupational environments also structured microbial exposure and host susceptibility. The eventual shift toward specific pathogen explanations required evidence that decoupled these mechanisms: water-source and exposure contrasts in cholera, laboratory isolation and culture, experimental inoculation, and later microbiological assays \citep{snow1855,pasteur1861,koch1884}. The lesson for mechanistic ML is therefore not ``collect big data,'' but ``collect the \emph{right} data to shrink $\mathcal M(P_X)$.''

This paper therefore proposes the following auditable standards as community expectations for any workflow, especially LLM-centered workflows, that present outputs as mechanistic. First, \textbf{Mechanism cards}: authors should combine every mechanistic explanation with a short structured appendix that (i) states the \emph{mechanism-side} assumptions (what is invariant under interventions, which functional forms/constraints are being imposed, what is excluded), (ii) states the \emph{observation-side} assumptions (how proxies relate to latent drivers; what is assumed stable vs.\ artifactual in measurement), and (iii) states what the assumptions \emph{actually identify}: which aspects of the mechanism are pinned down, and which remain an equivalence class (Appendix~\ref{app:ident_structure_def}). A mechanism card is incomplete without at least one explicit falsifier: an additional environment, intervention, new measurement channel, or invariance/derivative test that would \emph{rule out} the proposed mechanism. Second, \textbf{Multiplicity statements}: when identification is unavailable, the system must represent the identified set rather than output a single story. Concretely, it should report a set of distinct compatible mechanisms or a set of answers to the query $q$ (e.g., a sign-robust derivative range, a range of counterfactual effects), and it should label which qualitative conclusions are stable across that set. Third, \textbf{Mechanism-discriminating evaluation}: mechanistic claims must be evaluated on targets that directly probe identifying structure like interventions, invariances across environments, proxy-invariance checks tied to measurement, and derivative constraints.

If a paper or product claims a black box has ``discovered a mechanism,'' reviewers and readers should ask: what is assumed about the mechanism, what is assumed about the measurement channel, what data or contrasts actually discriminate among alternatives, and what else remains compatible? If the answer is unclear, the authors should phrase the results as bounded evidence: a coherent hypothesis, simulation, or partially identified set consistent with the observed proxies on their realized support, together with a clear statement of what the evidence rules out and what remains compatible.
In these settings, human experts remain central because the key tasks are precisely to state defensible assumptions, design discriminating experiments, and decide which claims are prediction-only versus mechanistic.

\section{Alternative Views}\label{sec:alt}
\paragraph{Alternative View 1: ``This is just epistemic uncertainty; ensembles and calibration solve it.''}
Ensembles and calibration are valuable, but they primarily address uncertainty about predictions under a fixed
observational target distribution. They do not, by themselves, identify mechanisms when proxies admit multiple
observationally compatible causal stories (Sections~\ref{sec:proxy_gap}--\ref{sec:cod}). The relevant ambiguity here
is \emph{structural identification uncertainty}: even with unlimited proxy data from the same measurement channel,
mechanistic queries \(q(M)\) can differ across \(M\in\mathcal M(P_X)\). Diversity helps precisely when it \emph{surfaces}
disagreement, but that requires an interface that does not collapse into one story.

\paragraph{Alternative View 2: ``Scaling and better data will select the true mechanism.''}
This objection is partly right: better data help \emph{when they add mechanistically informative variation}---new environments,
interventions, or measurement channels closer to the latent drivers---because these can shrink \(\mathcal M(P_X)\).
However, scale alone can also improve in-domain fit while leaving \(\mathcal M(P_X)\) large. In this framework, this
objection is best understood as a claim about the observational regime \(r\): if scaling implicitly changes \(r\)
(e.g., by adding environments or new measurement channels), then it is adding identifying structure. The request is
that such structure be made explicit and evaluated via discriminating tests, rather than assumed. Corollary~\ref{cor:scale_fixed_regime} makes this precise: if two observationally compatible mechanisms disagree on \(q\), no increase in sample size within the same observational regime can point-identify \(q\).
\paragraph{Alternative View 3: ``Identification is too strong a standard; useful science often proceeds with partial evidence.''}
This objection is right, and it is part of the position here rather than an exception to it. Many scientific and policy tasks do not require point-identifying a general mechanism. Partial identification, triangulation, process tracing, negative controls, and falsification tests can rule out candidate explanations, bound plausible effects, shift beliefs, guide data collection, and justify provisional action under uncertainty. The standard advocated here is narrower: when a workflow claims to have discovered a generalizable, intervention-stable mechanism, it should either supply the identifying structure that makes that claim answerable or report the remaining multiplicity. Evidence short of identification is valuable precisely when it is labeled as such. %

\paragraph{Alternative View 4: ``Narrative collapse is a UI issue, not a scientific issue.''}
The point is well taken: narrative collapse is partly an interface issue. When the interface encourages single-story explanations in underidentified settings, it invites
overclaiming. Related phenomena like model collapse and algorithmic monoculture can amplify the problem in ecosystems,
but narrative collapse can occur even in a single interaction on a fixed dataset. The appropriate response is not only
better UX; it is reporting and evaluation norms that force mechanistic claims to be conditional on discriminating tests
and to represent multiplicity when identification is unavailable.

\paragraph{Alternative View 5: ``Occam/simplicity will pick the right mechanism among observationally equivalent ones.''}
Simplicity biases (explicit or implicit) can be effective inductive biases in practice, and they often guide
scientific progress. The point here is that simplicity is itself an \emph{identifying assumption}: it selects one member of
a proxy-compatible equivalence class. When a workflow reports a single mechanism on the basis of simplicity, it should
declare (i) the simplicity criterion being used (description length, sparsity, functional form, etc.), (ii) which
alternatives remain compatible absent that criterion, and (iii) falsifiers or regime shifts that would discriminate
between the selected mechanism and plausible competitors. %

\paragraph{Alternative View 6: ``Mechanisms are only defined up to query-specific equivalence; reporting one representative is fine.''}
Mechanism claims are indeed query-specific: if two mechanistic stories imply the same answer to the intended
interventional query \(q\), then they are equivalent \emph{for that query}. The concern is precisely that LLM-centered
workflows tend to output a rich narrative that contains many \emph{additional} mechanistic commitments not warranted by
the evidence (or by the intended query), and to present those commitments as uniquely supported. A safer norm is to
separate (i) query-stable conclusions from (ii) narrative embellishments that vary across \(M\in\mathcal M(P_X)\).

\paragraph{When ML/LLM-centered workflows \emph{can} genuinely support discovery.}
This paper argues against treating fluent narratives or predictive success as evidence of mechanism discovery absent identifying structure, not against complex models. AlphaFold2, for example, pairs model complexity with a narrow target query, exact sequence inputs, evolutionary signal from alignments/templates, experimentally determined training structures, geometry-aware inductive biases, and blind evaluation against withheld structures \citep{jumper2021alphafold,nobel2024chemistry}. LLM-centered workflows can be constructive when they generate candidate mechanisms with explicit identifying assumptions, propose discriminating experiments, measurements, or invariance tests that shrink \(\mathcal M(P_X)\), and summarize evidence as identified sets rather than single narratives when identification is unavailable. %
Relatedly, recent work on open-ended autonomous discovery proposes steering hypothesis search by Bayesian surprise~\citep{agarwal2025autodiscovery}; the concern here is orthogonal: even principled exploration must either represent multiplicity or introduce discriminating structure when the same observational channel admits multiple incompatible mechanisms.

\section{Conclusion}
Without explicit identifying structure, it is not possible to tell whether a fluent model has learned a mechanism or merely a representation. The position here would be weakened by a demonstrated class of proxy-rich observational regimes in which
(i) mechanistic claims can be reliably distinguished, tightly bounded, or otherwise made robust using only declared observational structure,
and (ii) LLM-centered workflows reliably surface those limits, bounds, and falsifiers rather than collapsing
to a single narrative. High-performing predictive models can still be scientifically valuable for emulation, screening, design, measurement, and hypothesis generation even when they are not mechanistically faithful. The requirement is claim-evidence alignment: predictive or descriptive success should not be automatically elevated into mechanism discovery. Absent such evidence, the position is twofold: build models that predict well, and require mechanistic claims to be stated in a form that supports falsification, sensitivity analysis, and explicit accounting of multiplicity. Otherwise, the field may well end up ``discovering'' the next miasma.  

\section*{Acknowledgements}
Support for this research came from the Eunice Kennedy Shriver National Institute of Child Health and Human Development under Award Number R01HD107015 and Award Number R21HD119931, as well as from Grant N000142512270 from the Office of Naval Research.  Thanks to Izabel Aguiar, Arun Chandrasekhar, Jishnu Das, Avi Feller, Markus Goldstein, Rachel Heath, Jason Kerwin, Jeff Lockhart, Bodhisattwa Majumder, Harsh Parikh, Karl Rohe, Stephen Salerno, Elizabeth Stewart, and the anonymous referees for constructive comments on earlier versions of the draft.  All errors and shortcomings remain mine despite their guidance.  
\bibliographystyle{icml2026}
\bibliography{position_paper_arxiv_refs}

\newpage
\appendix
\onecolumn

\section{Summary Table of Running Examples}
\label{app:table}
Table~\ref{tab:running_examples} gives a summary table of the two examples used repeatedly throughout the paper.
\begin{table}[h!]
\centering
\small
\begin{tabular}{p{2.2cm} p{4.2cm} p{4.2cm} p{4.2cm}}
\toprule
\textbf{Case} & \textbf{Observed proxies $X$} & \textbf{Latent factors / mechanism $S$} & \textbf{Identifying structure (what shrinks $\mathcal M(P_X)$)} \\
\midrule
Mendel (peas) &
Discrete phenotypes (e.g., seed color/shape, flower color) across controlled crosses &
Genotypes/alleles; segregation and (approx.) independent assortment; genotype$\to$phenotype mapping &
Designed experiments (controlled mating, factorial crosses), targeted contrasts (F$_2$ ratios, test-crosses), and low-dimensional mechanistic parameterization that makes alternatives falsifiable \\
\addlinespace
Miasma vs germ-theoretic causation &
Environmental correlates: odor/``bad air'', visible fog, proximity to swamps/sewers, neighborhood incidence, weather/wind; coarse spatial-temporal aggregates &
Specific pathogen presence and transmission route (e.g., waterborne contamination), contact networks/exposure pathways; heterogeneous susceptibility &
Mechanism-discriminating measurements and contrasts: water-source comparisons, exposure histories, quasi-experimental exposure changes, microbiological assays, and environment shifts that decouple proxies from causes \\
\bottomrule
\end{tabular}
\caption{Running examples: proxies, latent mechanisms, and the extra structure required to identify mechanisms rather than merely fit observations.}
\label{tab:running_examples}
\end{table}

\section{Proofs of Propositions~\ref{thm:proxy_nogo}--\ref{prop:narrative}}
\label{app:proofs}

\noindent\textbf{Proof of Proposition~\ref{thm:proxy_nogo} (Proxy identifiability up to blockwise reparameterization).}
Fix two admissible deterministic representations of the same proxy blocks:
\[
X_b=f_b(S_b)\qquad\text{a.s.}
\qquad\text{and}\qquad
X_b=\tilde f_{\pi(b)}(\tilde S_{\pi(b)})\qquad\text{a.s.}
\]
for each $b=1,\dots,B$, where $\pi$ is a permutation of $\{1,\dots,B\}$ and each $f_b$ and $\tilde f_{\pi(b)}$ is a
bimeasurable bijection onto $\mathcal X_b^{\mathrm{supp}}$.

For each block $b$, define
\[
g_b:=\tilde f_{\pi(b)}^{-1}\circ f_b.
\]
Because $f_b$ and $\tilde f_{\pi(b)}$ are bimeasurable bijections, $g_b$ is a bimeasurable bijection from $\mathcal S_b$ onto
$\tilde{\mathcal S}_{\pi(b)}$, with inverse
\[
g_b^{-1}=f_b^{-1}\circ \tilde f_{\pi(b)}.
\]
Applying $\tilde f_{\pi(b)}^{-1}$ to the two representations of the same observed proxy block $X_b$ gives
\[
\tilde S_{\pi(b)}
\;=\;
\tilde f_{\pi(b)}^{-1}(X_b)
\;=\;
\tilde f_{\pi(b)}^{-1}(f_b(S_b))
\;=\;
g_b(S_b)
\qquad\text{a.s.}
\]
for each $b$. Rearranging the definition of $g_b$ yields
\[
\tilde f_{\pi(b)}=f_b\circ g_b^{-1}.
\]
Therefore any two admissible deterministic representations of the same proxy blocks differ only by blockwise bimeasurable
reparameterizations, together with a permutation when the block labels are unlabeled.

Finally, because every admissible deterministic representation of an observational law $P_X$ can be realized by taking
$X\sim P_X$ and setting $S_b=f_b^{-1}(X_b)$ blockwise, the same conclusion characterizes the deterministic observational
equivalence class associated with $P_X$. Hence, within this admissible deterministic class, the latent blocks are identified
exactly up to within-block bimeasurable reparameterization and block permutation.
\hfill\(\square\)

\medskip
Why is this a problem? Without additional anchors, the coordinates of a recovered latent block have no stable scientific meaning by themselves: \(S_b\) and \(\log S_b\) can represent the same evidence under a compensating measurement map, yet imply different claims about linearity, additivity, or intervention magnitude. Only reparameterization-invariant features are identified absent further structure.

\medskip
\noindent\textbf{Proof of Proposition~\ref{thm:proxy_nogo_channel} (Non-identification persists under noise).}
Let $P_S$ denote the joint distribution of $S=(S_1,\dots,S_B)$. The proxy distribution induced by $(P_S,K)$ is, for any measurable
$A$ in the proxy space,
\[
P_X(A)
\;=\;
\int K(A\mid s)\,P_S(ds).
\]
Assume $K(x\mid s)=\prod_{b=1}^B K_b(x_b\mid s_b)$ (conditionally independent channels).
Let $g(s):=(g_1(s_1),\dots,g_B(s_B))$ with each $g_b$ bimeasurable, and define the reparameterized latent variable
$\tilde S := g(S)$ (so $P_{\tilde S}=P_S\circ g^{-1}$). Define the pulled-back channels
\[
\tilde K_b(\,\cdot \mid \tilde s_b)
\;:=\;
K_b(\,\cdot \mid g_b^{-1}(\tilde s_b)),
\qquad\text{and}\qquad
\tilde K(\cdot\mid \tilde s)
\;:=\;
\prod_{m=1}^B \tilde K_b(\cdot\mid \tilde s_b).
\]
Because each $g_b^{-1}$ is measurable and each $K_b$ is a Markov kernel, each $\tilde K_b$ is again a Markov kernel.
Then for any measurable set $A$,
\[
\tilde P_X(A)
\;:=\;
\int \tilde K(A\mid \tilde s)\,P_{\tilde S}(d\tilde s).
\]
Using the defining property of a pushforward measure,
\[
\int \varphi(\tilde s)\,P_{\tilde S}(d\tilde s)
\;=\;
\int \varphi(g(s))\,P_S(ds)
\qquad\text{for any measurable }\varphi,
\]
and taking $\varphi(\tilde s)=\tilde K(A\mid \tilde s)$ yields
\[
\tilde P_X(A)
\;=\;
\int \tilde K(A\mid g(s))\,P_S(ds)
\;=\;
\int K(A\mid s)\,P_S(ds)
\;=\;
P_X(A),
\]
where the middle equality holds because
\[
\tilde K(A\mid g(s))
\;=\;
\prod_{b=1}^B \tilde K_b(A_b\mid g_b(s_b))
\;=\;
\prod_{b=1}^B K_b(A_b\mid s_b)
\;=\;
K(A\mid s)
\]
(first for measurable rectangles $A=\prod_b A_b$, then extended to the full $\sigma$-algebra by standard monotone class arguments).
Hence the induced proxy law $P_X$ is unchanged by blockwise invertible reparameterizations paired with the pulled-back channels,
so the same blockwise non-identification persists under arbitrary conditionally independent measurement noise.
\hfill\(\square\)

\medskip
\noindent\textbf{Proof of Proposition~\ref{prop:narrative} (Narrative collapse as minimax ambiguity).}
Let $M_1,M_2\in\mathcal M(P_X)$ satisfy $q(M_1)\neq q(M_2)$. Write $a_1:=q(M_1)$ and $a_2:=q(M_2)$ with $a_1\neq a_2$.
Because $M_1,M_2\in\mathcal M(P_X)$, both mechanisms induce the same proxy distribution $P_X$, and therefore induce the same
distribution over datasets $D=(X^{(1)},\dots,X^{(n)})$, namely $P_X^n$.

Fix any (measurable) single-valued explanation rule $\widehat a=\widehat a(D)$. Since $D\sim P_X^n$ under both $M_1$ and $M_2$,
all probabilities below are with respect to the same distribution $P_X^n$.
Define
\[
p_1 := \mathbb P\bigl(\widehat a(D)=a_1\bigr),
\qquad
p_2 := \mathbb P\bigl(\widehat a(D)=a_2\bigr).
\]
Because $a_1\neq a_2$, the events $\{\widehat a=a_1\}$ and $\{\widehat a=a_2\}$ are disjoint, hence $p_1+p_2\le 1$.

Under 0--1 loss $\ell(\widehat a,q(M))=\mathbf 1\{\widehat a\neq q(M)\}$, the risks at $M_1$ and $M_2$ are
\[
\mathbb E\bigl[\ell(\widehat a(D),q(M_1))\bigr]
\;=\;
\mathbb P(\widehat a\neq a_1)
\;=\;
1-p_1,
\qquad
\mathbb E\bigl[\ell(\widehat a(D),q(M_2))\bigr]
\;=\;
1-p_2.
\]
Therefore,
\[
\max\Bigl\{\,\mathbb E[\ell(\widehat a(D),q(M_1))],\ \mathbb E[\ell(\widehat a(D),q(M_2))]\,\Bigr\}
\;\ge\;
\frac{(1-p_1)+(1-p_2)}{2}
\;=\;
1-\frac{p_1+p_2}{2}
\;\ge\;
\frac{1}{2}.
\]
Since $\{M_1,M_2\}\subseteq \mathcal M(P_X)$, it follows that
\[
\sup_{M\in\mathcal M(P_X)} \mathbb E\bigl[\ell(\widehat a(D),q(M))\bigr]
\;\ge\;
\max\Bigl\{\,\mathbb E[\ell(\widehat a(D),q(M_1))],\ \mathbb E[\ell(\widehat a(D),q(M_2))]\,\Bigr\}
\;\ge\;
\frac{1}{2}.
\]
Taking the infimum over all single-valued rules $\widehat a$ yields the stated minimax lower bound:
\[
\inf_{\widehat a}\ \sup_{M\in\mathcal M(P_X)}\ \mathbb E\bigl[\ell(\widehat a(D),q(M))\bigr]
\;\ge\;
\frac{1}{2}.
\]
In particular, no single-valued rule can be uniformly correct over $\mathcal M(P_X)$ whenever two observationally compatible
mechanisms disagree on the query.
\hfill\(\square\)

\begin{corollary}[Scale cannot resolve fixed-regime non-identification]\label{cor:scale_fixed_regime}
Let $D_n\sim P_O^n$ be data from a fixed observational regime. If there exist $M_1,M_2\in\mathcal M(P_O)$ such that $q(M_1)\neq q(M_2)$, then for any estimator $\hat q_n=\hat q_n(D_n)$,
\[
\sup_{M\in\{M_1,M_2\}} \mathbb P_M\{\hat q_n(D_n)\neq q(M)\}\ge \tfrac12
\]
for every $n$.
\end{corollary}
\noindent\textbf{Proof.}
The two observational regimes are identical, so $D_n$ has the same distribution under $M_1$ and $M_2$. Any estimator therefore takes the same distribution of values under both mechanisms. Since $q(M_1)\neq q(M_2)$, the estimator cannot be correct with probability greater than \(1/2\) on both simultaneously. Taking the worst case over $\{M_1,M_2\}$ gives the bound.
\hfill\(\square\)

\section{Off-Support Non-Identification Under Thin Support}
\label{app:stronger}

\textbf{Proposition A.1 (Essential non-identification under observational support).}
Let \(\mathcal{X}\subset\mathbb{R}^d\) be a closed embedded \(k\)-dimensional submanifold with \(k<d\), and assume the observational distribution of \(X\) is supported on \(\mathcal{X}\). Let \(\mathcal{F}\) be a class of scalar-valued functions \(f:\mathbb R^d\to\mathbb R\) closed under smooth, compactly supported perturbations. Then for any \(f\in\mathcal{F}\) and any open set \(U\) with \(U\cap\mathcal{X}=\emptyset\), there exists \(g\in\mathcal{F}\) such that \(f=g\) on \(\mathcal{X}\) but \(f\neq g\) on \(U\). Hence \(f\) is not identified from observational data on \(\mathcal{X}\) without additional assumptions that constrain off-support behavior.

\textbf{Proof.}
Because \(U\) is open, pick an open ball \(B\subset U\). Since \(U\cap \mathcal X=\emptyset\),
then \(B\cap \mathcal X=\emptyset\). Let \(\phi\) be a smooth bump function supported on \(B\) with \(\phi(x_0)=1\) at some \(x_0\in B\). Define \(g=f+c\phi\) for any \(c\neq 0\). Then \(g=f\) on \(\mathcal{X}\) and \(g(x_0)\neq f(x_0)\). Closure under perturbations implies \(g\in\mathcal{F}\). \hfill\(\square\)

The assumption on $\mathcal F$ is a convenient sufficient condition; any hypothesis class that can represent (or approximate)
localized perturbations exhibits the same qualitative non-identification. This includes rich nonparametric classes and, in practice,
overparameterized neural networks that can approximate localized perturbations on off-support regions.

\section{Formal Comparison with Rashomon and Uncertainty}
\label{app:compare}

Let $S$ denote latent mechanism-level variables, let $X$ denote observed proxies generated by a (possibly stochastic) measurement channel $K(\cdot\mid S)$, and let $Y$ denote an outcome generated by some structural law $Y\sim P(\cdot\mid S)$ or, at the observational level, by the induced conditional $P(Y\mid X)$.
The claim concerns \emph{mechanism equivalence under proxies}: for a fixed proxy observational law $P_X$, define the \emph{mechanism compatibility set}
$
\mathcal M(P_X)\;=\;\bigl\{\,M=(P_S,K)\;:\; S\sim P_S,\ X\mid S\sim K(\cdot\mid S)\ \text{and the induced marginal is }P_X\,\bigr\}.
$
When $\lvert\mathcal M(P_X)\rvert>1$, distinct mechanisms $M_1\neq M_2$ (different latent stories and/or measurement channels) induce the same proxy evidence; thus mechanistic queries $q(M)$ (e.g., counterfactual responses or derivative constraints) are not identified from proxies without additional structure.
This is different from the \emph{Rashomon effect}, which is \emph{model-level} multiplicity under a fixed predictive target: for a hypothesis class $\mathcal H$ and risk $R(h)=\mathbb E[\ell(h(X),Y)]$, the Rashomon set is
$
\mathcal R_\varepsilon \;=\;\{\,h\in\mathcal H:\ R(h)\le \inf_{h'\in\mathcal H}R(h')+\varepsilon\,\},
$
capturing that many predictors achieve essentially the same predictive risk.
Rashomon multiplicity lives to the right of \(X\); the proxy gap lives to the left of \(X\).
It is also distinct from \emph{underspecification}, which emphasizes that many predictors can match in-domain objectives yet behave differently under deployment, distribution shift, or interventions: there exist $h_1,h_2\in\mathcal R_\varepsilon$ such that $h_1(X)\neq h_2(X)$ on $X$ drawn from a perturbed or shifted environment.
In short, Rashomon and underspecification describe multiplicity \emph{within} a chosen predictor class $\mathcal H$ for a fixed
observational task (and thus can shrink if one changes $\mathcal H$ or adds regularization), whereas the proxy-gap claim is
\emph{mechanism-level non-identification}: even at the population level, and even with infinite data generated through the same
measurement channel $K(\cdot\mid S)$ (so $P_X$ and even $P(Y\mid X)$ are known exactly), there can remain multiple distinct
$(P_S,K)$ (hence multiple incompatible mechanistic answers $q(M)$) consistent with the same proxy evidence.
Causal-quartet constructions provide concrete analogues: identical summaries can require different causal adjustment decisions \citep{mcgowan2024causal}, and identical average effects can mask different latent treatment-effect patterns \citep{gelman2024causalquartets}.

\section{Additional Pea Plant Details and Results}
\label{app:pea}
Each MLP is trained with early stopping (patience \(=10\)) and maximum iterations \(=250\), using random seeds \(0,1,\dots,59\) (so 60 total fits). The analysis evaluates predictive performance using held-out cross-entropy (log loss) on a 25\% test split.  Figure~\ref{fig:peas_regimeA} visualizes the central identification issue in Regime~A: when evidence is reduced to pooled F$_2$ phenotype proxies, mechanistically distinct explanations can be observationally equivalent on the support of the data. In particular, the Mendelian dihybrid mechanism (dominance with independent segregation of the two traits) and the ``independent traits'' alternative (fixed marginals \(p_R=p_Y=3/4\)) induce essentially the same pooled proxy distribution, so the proxy channel alone does not select a unique mechanism. Figure~\ref{fig:peas_regimeB} shows how this equivalence breaks once experimental structure enters: the alternative mechanism is ruled out because it cannot simultaneously explain the outcomes of multiple designed crosses, while the Mendelian mechanism remains well calibrated across \texttt{F2\_self}, \texttt{testcross}, \texttt{monoA}, and \texttt{monoB}.

Table~\ref{tab:peas_concrete_query} gives the concrete prediction behind the spread plot. For the testcross \texttt{RY}\(\times\)\texttt{WG}, Mendelian inheritance predicts equal probabilities for the four offspring phenotypes. Without cross-type information, the MLP mean assigns nearly half its mass to \texttt{RY}; with testcross labels and support, the predictions move close to the Mendelian answer. This is not evidence for a competing biological law, but a symptom of asking a phenotype-only F$_2$ predictor to answer a counterfactual that its evidence regime does not identify.

\begin{table}[h!]
\centering
\small
\caption{\textbf{Concrete counterfactual prediction for \texttt{RY}\(\times\)\texttt{WG} with \texttt{cross=testcross}.}
The Mendelian testcross answer assigns probability \(1/4\) to each offspring phenotype. The MLP rows report means over near-optimal fits. In the phenotype-only F$_2$ regime, MLPs are forced to answer an underidentified query and produce an unsupported counterfactual distribution. In the design regime with testcross support and labels, predictions move close to the Mendelian target.}
\begin{tabular}{lcccc}
\toprule
 & \texttt{RG} & \texttt{RY} & \texttt{WG} & \texttt{WY} \\
\midrule
Mendelian testcross & .25 & .25 & .25 & .25 \\
No-design MLP mean & .20 & .49 & .10 & .20 \\
With-design MLP mean & .26 & .26 & .24 & .24 \\
\bottomrule
\end{tabular}
\label{tab:peas_concrete_query}
\end{table}

\paragraph{Counterfactual Dispersion and Its Collapse Under Design.}
Figures~\ref{fig:peas_cf_nodesign} and~\ref{fig:peas_cf_withdesign} provide an intuition for the spread statistic reported in the main text. Each point cloud displays the predicted offspring phenotype distribution for the same observed parent phenotype pair \texttt{RY}\(\times\)\texttt{WG} across many independently trained MLP fits (restricted to near-optimal seeds). Without cross-type labels and testcross support (no design), this counterfactual query is underdetermined: multiple equally predictive models produce noticeably different probability vectors over \{\texttt{RG},\texttt{RY},\texttt{WG},\texttt{WY}\}. When the design regime supplies the cross label and testcross support, the query becomes well-posed (here, \texttt{cross=testcross}), and the counterfactual predictions concentrate, reflecting a smaller mechanism-compatible equivalence class.

\paragraph{Why Restrict to Near-Optimal Seeds.}
Figure~\ref{fig:peas_blackbox_hist} shows the distribution of held-out log loss across MLP random seeds and highlights the near-optimal subset used in Figures~\ref{fig:peas_cf_nodesign}--\ref{fig:peas_cf_withdesign} and in the spread-collapse summary. This addresses the concern that mechanistic disagreement is an optimization artifact: even among predictors with essentially indistinguishable predictive performance, counterfactual mechanistic predictions can diverge when identifying structure is absent, and concentrate when that structure is restored.

\begin{figure}[h!]
  \centering
  \includegraphics[width=.75\linewidth]{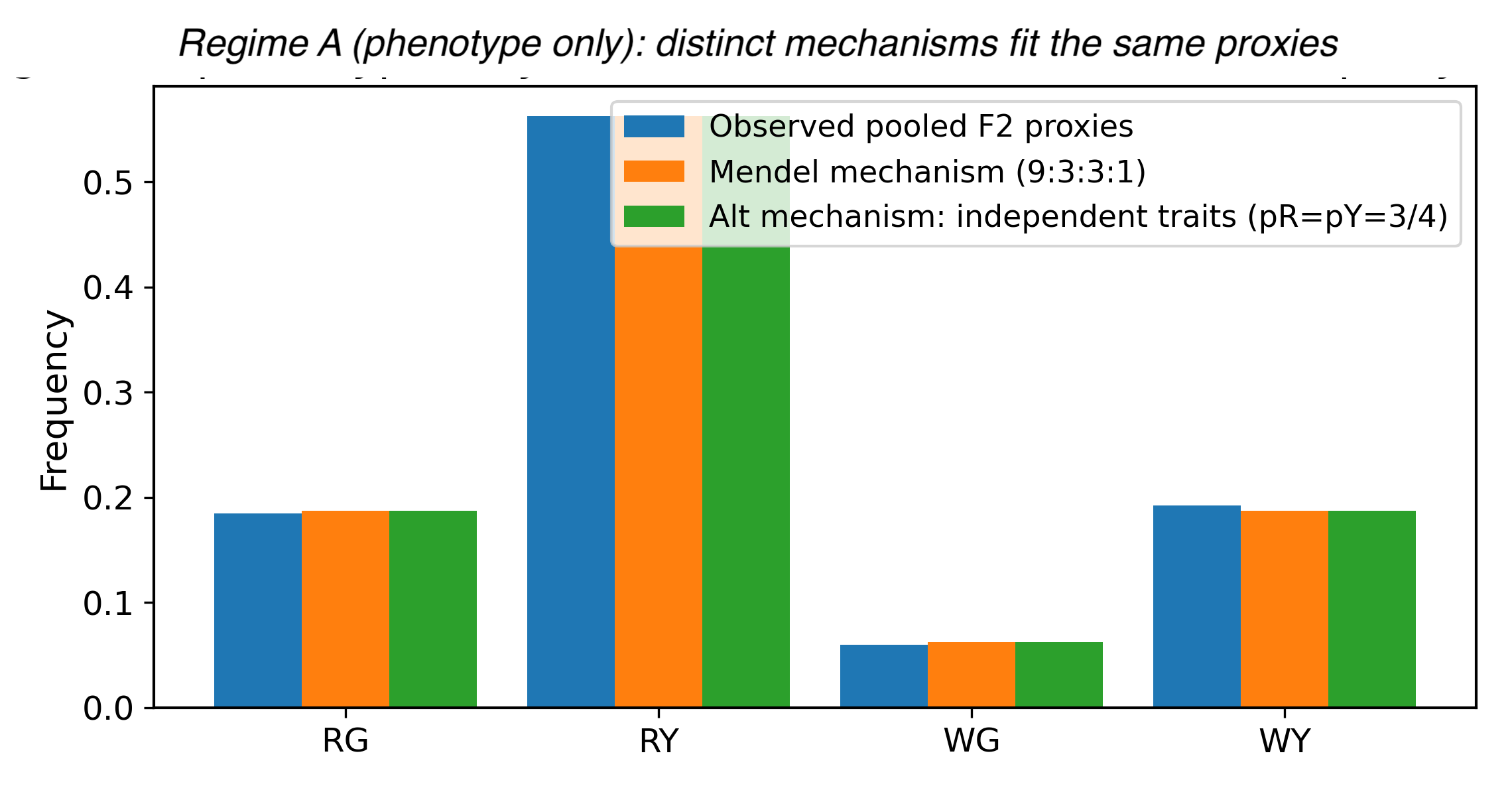}
  \caption{\textbf{Regime A (phenotype-only): observational equivalence at the proxy level.}
  Pooled F$_2$ phenotype frequencies over four proxy categories (\texttt{RG}, \texttt{RY}, \texttt{WG}, \texttt{WY}). The Mendelian dihybrid mechanism (dominance + independent assortment; \(9{:}3{:}3{:}1\)) and an alternative ``independent traits'' mechanism with fixed marginals \(p_R=p_Y=3/4\) are distinct mechanistic stories but induce essentially the same pooled proxy distribution in this regime.}
  \label{fig:peas_regimeA}
\end{figure}

\begin{figure}[h!]
  \centering
  \includegraphics[width=.75\linewidth]{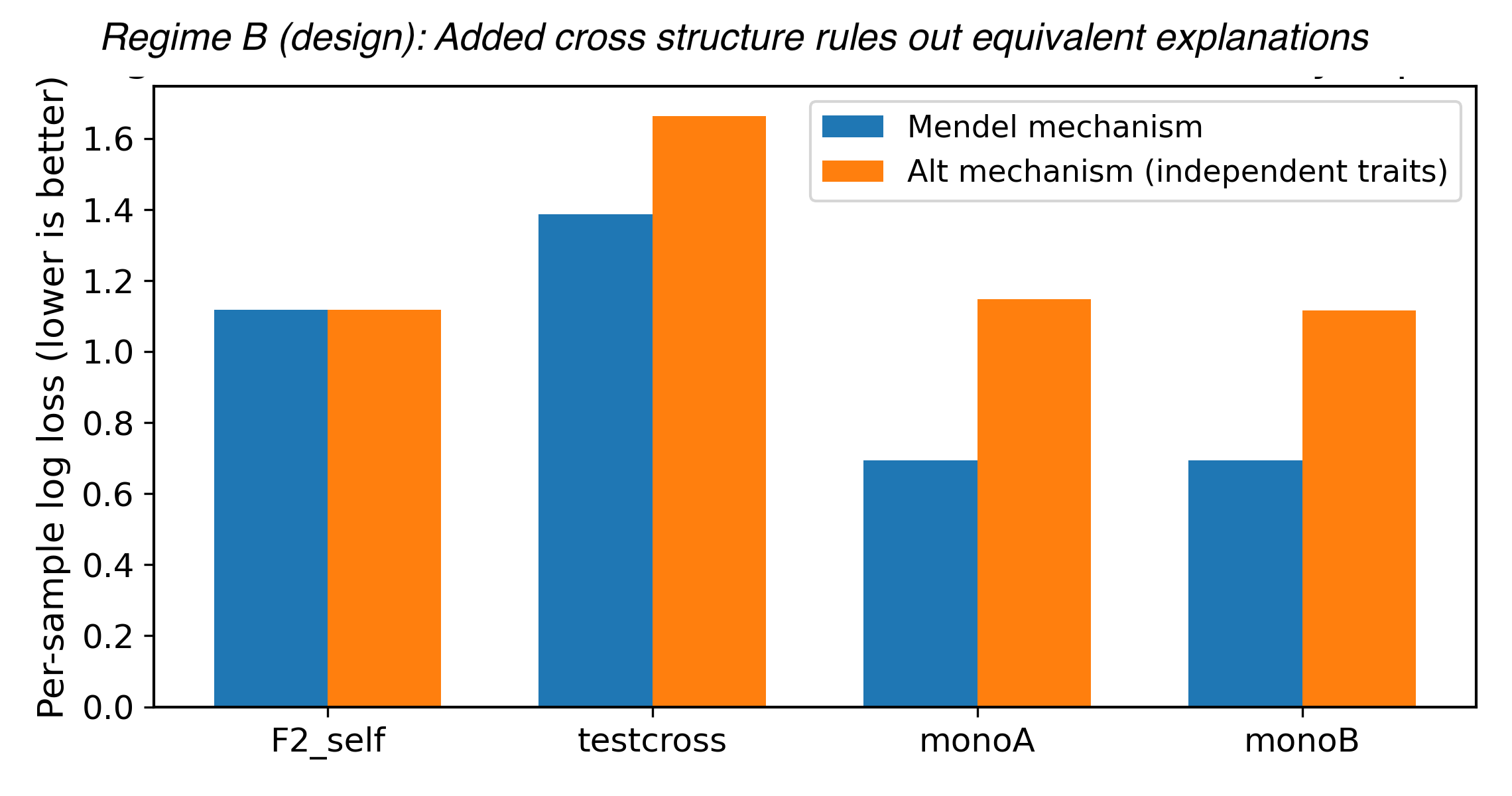}
  \caption{\textbf{Regime B (design): cross structure rules out an observationally equivalent alternative.}
  Per-sample log loss under additional labeled crosses (\texttt{F2\_self}, \texttt{testcross}, \texttt{monoA}, \texttt{monoB}). The alternative mechanism matches Regime~A by construction but fails when asked to simultaneously explain the outcomes of multiple designed crosses; the Mendelian mechanism remains well-calibrated across all crosses.}
  \label{fig:peas_regimeB}
\end{figure}

\begin{figure}[h!]
  \centering
  \includegraphics[width=.75\linewidth]{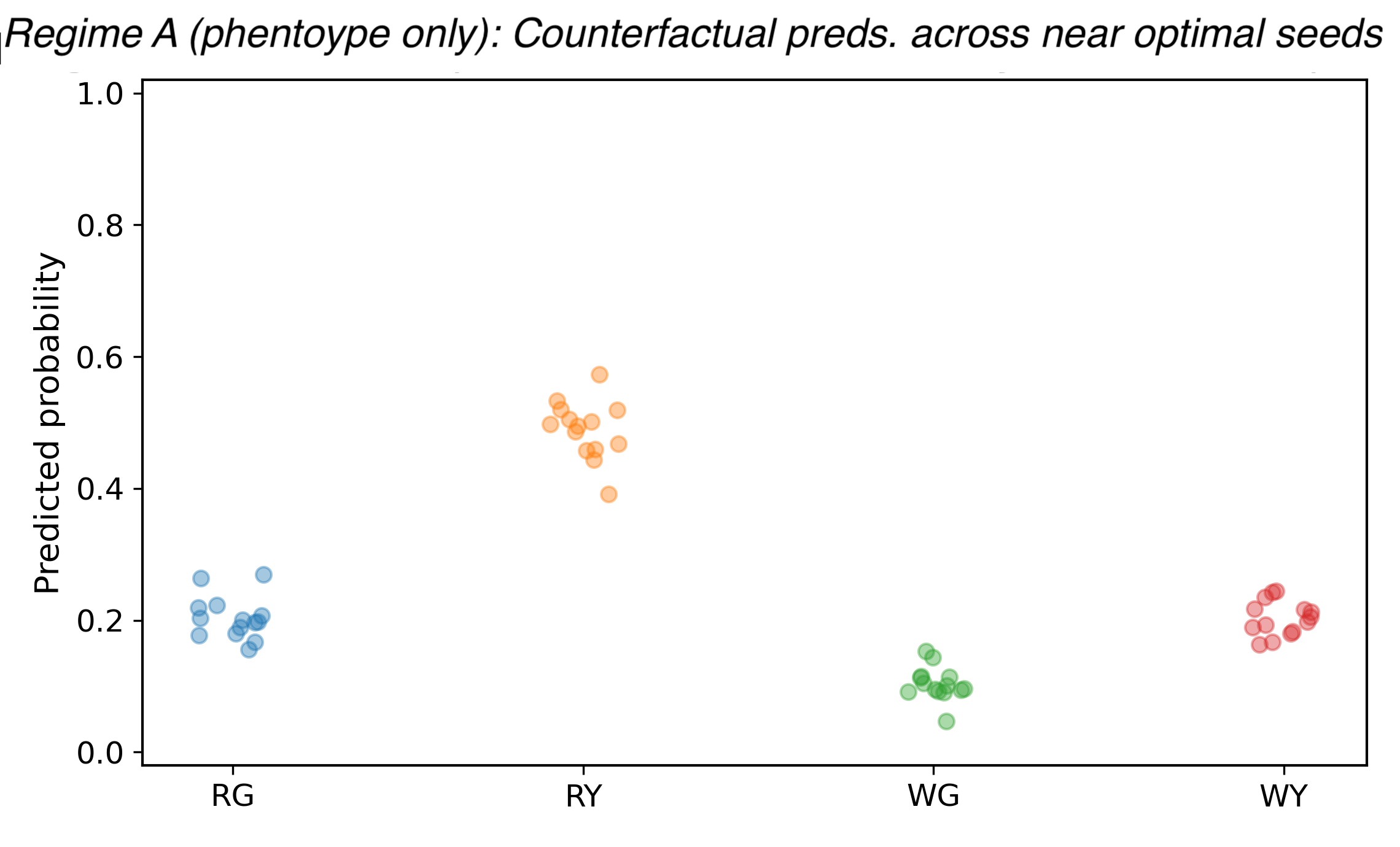}
  \caption{\textbf{No design label: counterfactual predictions vary across near-optimal black-box fits.}
  Each point is the predicted probability of a phenotype category under a counterfactual query (\texttt{RY}\(\times\)\texttt{WG}) from a different near-optimal MLP fit trained without cross-type labels. Even among near-optimal predictors, counterfactual mechanistic predictions exhibit nontrivial spread when identifying structure is absent.}
  \label{fig:peas_cf_nodesign}
\end{figure}

\begin{figure}[h!]
  \centering
  \includegraphics[width=.75\linewidth]{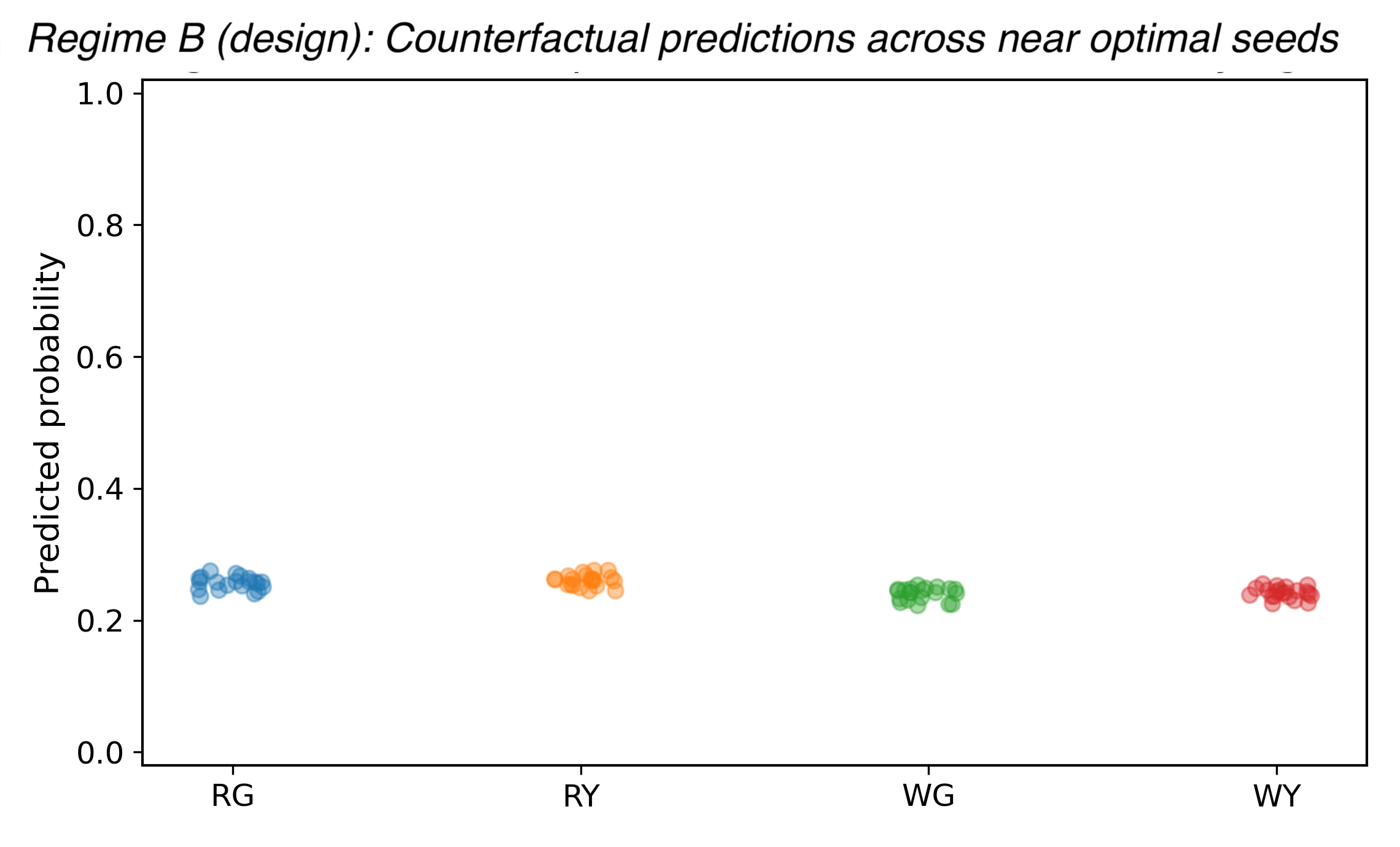}
  \caption{\textbf{With design label: counterfactual predictions concentrate.}
  Same counterfactual parent phenotype pair as Figure~\ref{fig:peas_cf_nodesign}, but now the query is well-posed by conditioning on the experimental design label (\texttt{cross=testcross}) and models are trained with cross labels. Counterfactual predictions concentrate substantially across near-optimal seeds, reflecting reduced mechanistic ambiguity.}
  \label{fig:peas_cf_withdesign}
\end{figure}

\begin{figure}[h!]
  \centering
  \includegraphics[width=.75\linewidth]{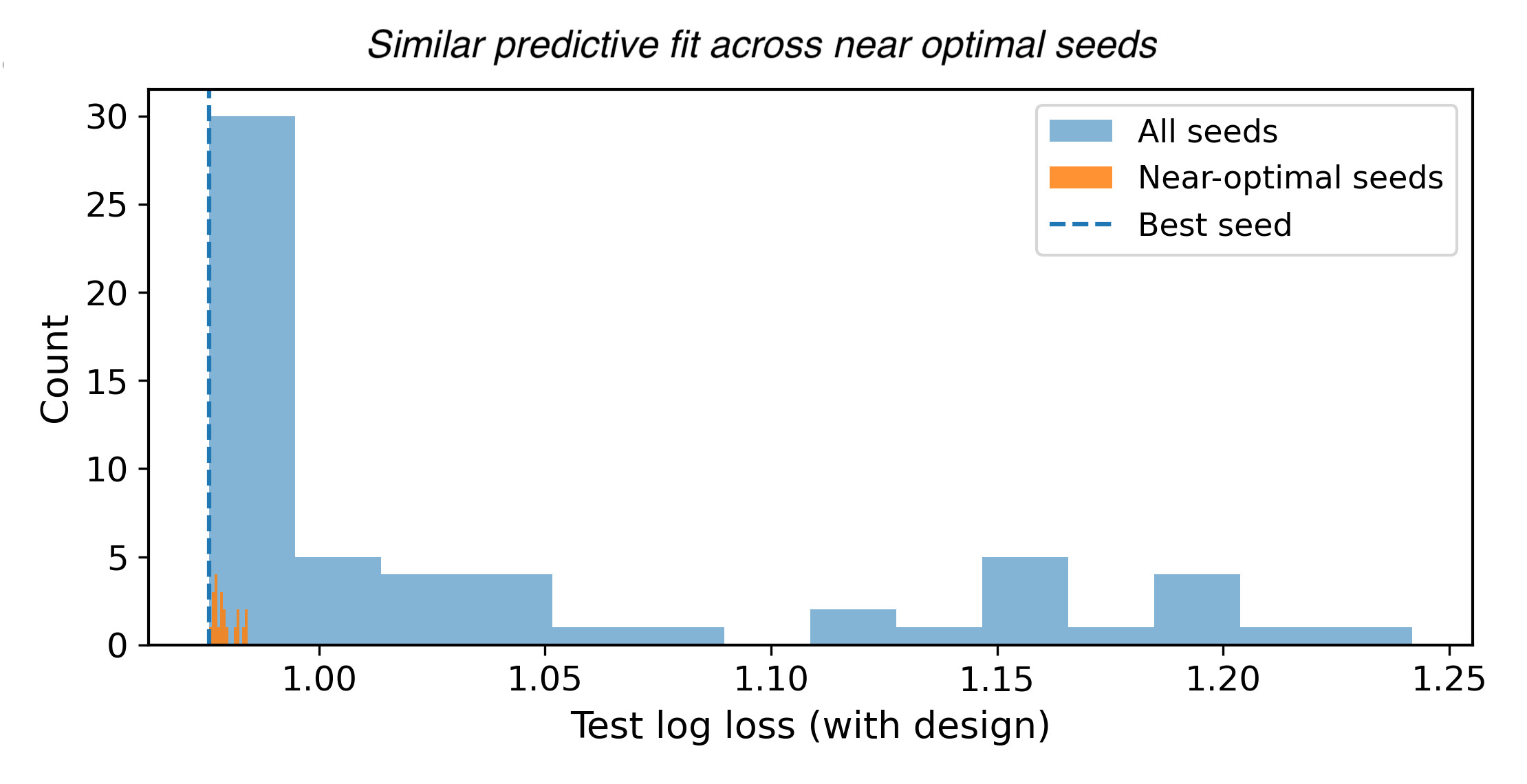}
  \caption{\textbf{Black-box multiplicity: many near-optimal fits.}
  Distribution of held-out test log loss across MLP random seeds when trained with design labels; near-optimal seeds are highlighted. This justifies evaluating mechanistic disagreement \emph{within} a set of comparably good predictors, rather than attributing disagreement to poor optimization.}
  \label{fig:peas_blackbox_hist}
\end{figure}

\section{Simulation Study: In-Domain Agreement, Mechanistic Disagreement}\label{app:sim}

The appendix includes a minimal synthetic study to illustrate that (i) many flexible models can achieve almost identical in-domain prediction on high-dimensional proxies while (ii) disagreeing substantially on mechanistic quantities (derivatives and off-support responses).

\paragraph{Data-Generating Process.}
The data-generating process draws $x_1\sim\mathcal N(0,1)$; sets \(x_2=x_1+\epsilon_{x,2}\) with \(\epsilon_{x,2}\sim\mathcal N(0,\sigma_x^2)\); creates correlated proxy blocks
\(x_3,\dots,x_{d-d_{\text{noise}}}\) via \(x_j=\sin(x_1)+0.3\cos(0.5x_1)+\epsilon_{x,j}\) with
\(\epsilon_{x,j}\overset{i.i.d.}{\sim}\mathcal N(0,\sigma_x^2)\); and adds $d_{\text{noise}}$ pure-noise proxies $\mathcal N(0,1)$. The simulation uses $d=20$, $d_{\text{noise}}=6$, $\sigma_x=0.15$, and $y=\sin(x_1)+0.1x_2+\epsilon_y$ with $\epsilon_y\sim\mathcal N(0,0.05^2)$. Train/test sizes are $n_{\text{train}}=2000$, $n_{\text{test}}=1000$.

\paragraph{Models.}
The simulation trains three model families across independent random seeds: (i) MLP (two hidden layers of width 64, ReLU, Adam, 80 epochs, LR $2\times10^{-3}$), (ii) Random Forest regressor (300 trees, min leaf 5), and (iii) Gradient Boosting regressor (400 trees, depth 3, LR 0.05). Counts: 10 MLPs, 8 RFs, 4 GBRs. All runs use the same data; variation arises only from random seeds (initialization, sampling order, and tree randomness).

\paragraph{Mechanistic Probes.}
For each trained model the analysis computes (a) test MSE on an i.i.d.\ hold-out set; (b) partial derivatives $\partial\hat y/\partial x_j$ at 80 test points (autograd for MLPs, finite differences for tree models), summarized by $\mathrm{Var}(\partial\hat y/\partial x_j)$ across models; and (c) an off-manifold intervention slice where $x_2$ is varied in $[-3,3]$ while other coordinates are held fixed. True gradients are $\partial y/\partial x_1=\cos(x_1)$, $\partial y/\partial x_2=0.1$, and $0$ for $j\ge3$, so variance of order $10^{-3}$ or below is practically small, while variance $\gtrsim 10^{-2}$ implies meaningfully different mechanistic stories. In the saved simulation output, the displayed gradient-variance values range from roughly \(2.4\times10^{-5}\) to \(1.2\times10^{2}\), with most values near \(10^{-3}\) and a heavy upper tail driven by tree-model finite-difference gradients on mechanism-bearing and proxy coordinates.

\paragraph{Results.}
\begin{figure}[h!]
  \centering
  \includegraphics[width=0.75\linewidth]{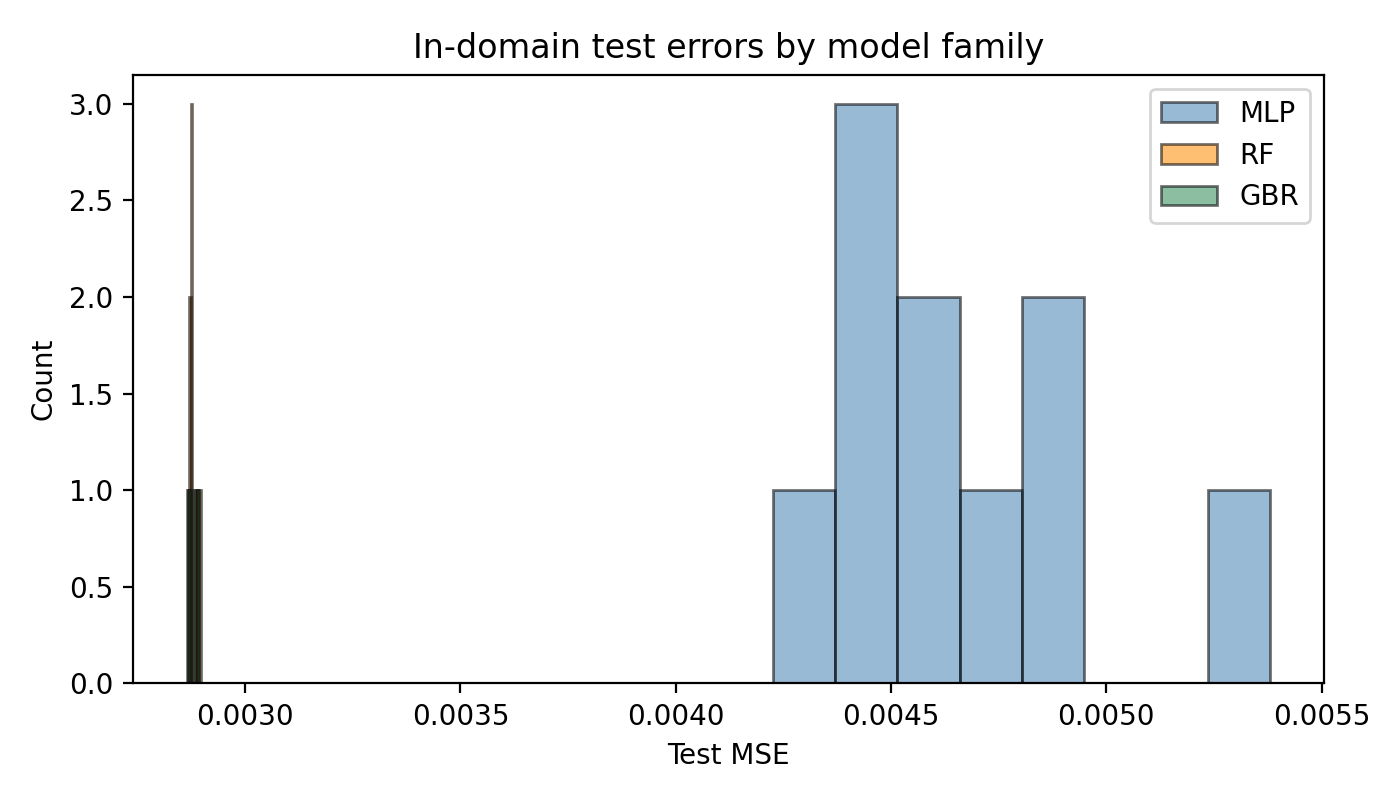}
  \caption{\textbf{In-domain fit.} Histogram of test MSE over all seeds and model families. Despite different inductive biases, all models cluster tightly ($\mathrm{MSE}\approx 0.003$--$0.005$), showing strong in-domain agreement.}
  \label{fig:hist_mse}
\end{figure}

\begin{figure}[h!]
  \centering
  \includegraphics[width=0.68\linewidth,height=0.72\textheight,keepaspectratio]{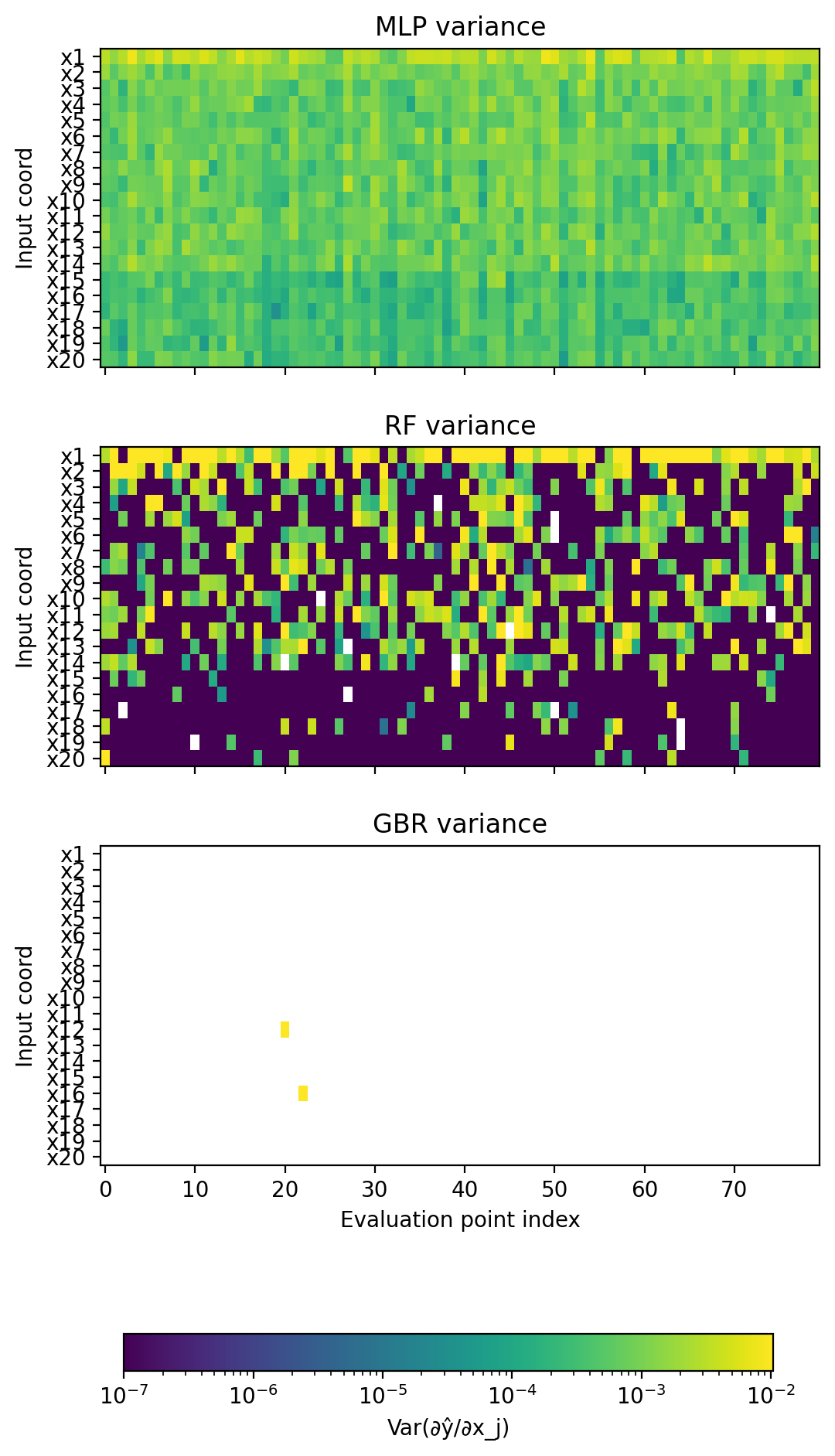}
  \caption{\textbf{Gradient disagreement by family.} Variance of partial derivatives $\partial\hat y/\partial x_j$ across seeds within each model family, log-scaled. Dark indicates agreement; bright indicates disagreement. The colorbar (bottom) is the scale for all panels. RF and GBR show sizable variance on $x_1$ and $x_2$ (mechanism-bearing coordinates) and spurious variance on proxy/noise coordinates; MLPs exhibit lower but nonzero variance.}
  \label{fig:grad_var_family}
\end{figure}

\begin{figure}[h!]
  \centering
  \includegraphics[width=0.75\linewidth]{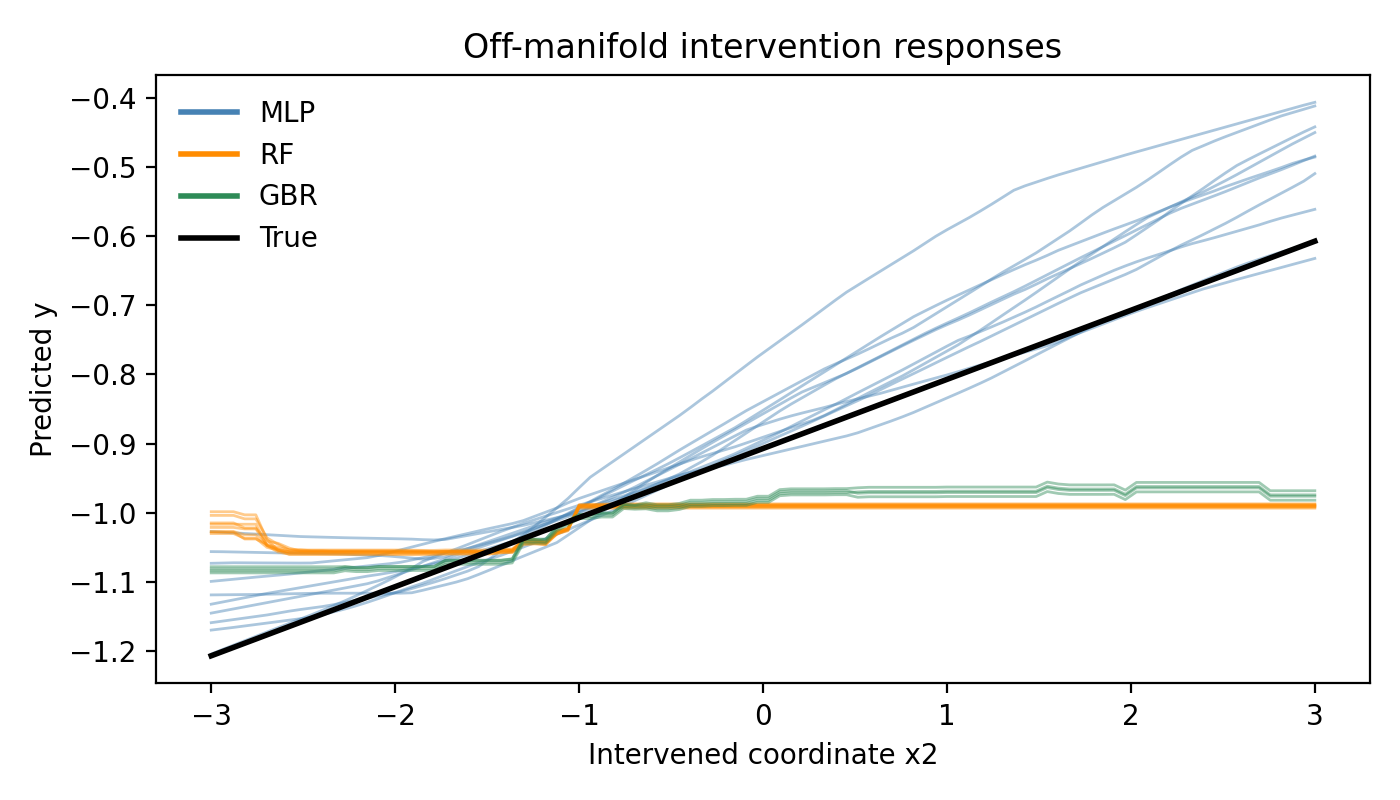}
  \caption{\textbf{Off-manifold interventions.} Each gray line is one model’s prediction when intervening on $x_2$ (holding other coordinates fixed); black is the true response. Although in-domain MSEs are nearly identical (Figure~\ref{fig:hist_mse}), models diverge markedly off the observed manifold, underscoring underdetermination of mechanism.}
  \label{fig:off_manifold}
\end{figure}

The key pattern is consistent with this paper's claims: tight predictive clustering (Figure~\ref{fig:hist_mse}) coexists with substantial gradient variance (Figure~\ref{fig:grad_var_family}) and divergent off-manifold responses (Figure~\ref{fig:off_manifold}). Because the only differences across runs are random seeds, these disagreements are attributable to underdetermination given proxies and model flexibility, not to differences in data or hyperparameters. The gradient variances provide an interpretable scale: most true gradients are $\mathcal O(10^{-1})$, so observed variances $>10^{-2}$ imply materially different sensitivities and therefore different mechanistic stories, even when predictions on the observed support agree.

\begingroup
\footnotesize
\section{Six Worked Examples of Identifying Structure}
\label{app:ident_structure_examples}

This appendix instantiates the definition in Appendix~\ref{app:ident_structure_def} in six settings that are representative of
modern ``mechanistic ML.'' Each example makes the same decomposition explicit:
\[
\text{mechanism-side restrictions}\quad(\mathcal A_{\mathrm{mech}})\qquad
\text{observation-side restrictions}\quad(\mathcal A_{\mathrm{obs}})\qquad
\text{design / regimes}\quad(r),
\]
and points to a mechanistic query $q(M)$ that becomes identifiable only after the relevant structure is imposed.

\vspace{0.5ex}
\noindent\textbf{Example 1: Disentanglement versus identifiable latent variable models.}
Locatello et al.\ show that unsupervised disentanglement is generically non-identifiable without additional inductive bias or
supervision; Khemakhem et al.\ show that identifiability can be restored by adding an observed auxiliary variable and restricting
how it modulates the latent distribution \citep{pmlr-v97-locatello19a,pmlr-v108-khemakhem20a}.
\[
\begin{aligned}
\textbf{Latent mechanism }S &: S=(S_1,\dots,S_d)\ \text{(``factors'')},\\
\textbf{Proxies }X &: X=f(S)\ \text{(unknown nonlinear mixing)},\\
\mathcal A_{\mathrm{mech}} &: \text{(baseline)}\ S_1,\dots,S_d\ \text{independent (factorized prior)},\\
\mathcal A_{\mathrm{obs}} &: \text{(baseline)}\ f\ \text{invertible/smooth (a broad class)},\\
r &: \text{(baseline)}\ \text{single environment; observe only }X.\\[0.5ex]
q(M) &: \text{``recover the latent factors'' (up to permutation/componentwise transforms).}
\end{aligned}
\]
Under the baseline regime and assumptions, the equivalence class is large: many distinct latent representations induce the same
$P_X$ (non-identification) \citep{pmlr-v97-locatello19a}. Identification is recovered by strengthening both observation and design:
\[
\begin{aligned}
\mathcal A_{\mathrm{mech}}' &: P_{S\mid U}(s\mid u)=\prod_{j=1}^d \exp\{\eta_j(u)^\top T_j(s_j)-A_j(\eta_j(u))\}\ \text{(factorized exponential family)},\\
\mathcal A_{\mathrm{obs}}' &: X=f(S)\ \text{with suitable regularity/invertibility},\\
r' &: \text{observe an auxiliary variable }U\ \text{(labels/environments/context) that varies }\eta(u).
\end{aligned}
\]
Then latent variables become identifiable up to simple transformations (the precise equivalence depends on assumptions), turning an
ill-posed representation-learning task into an identification problem with explicit structure \citep{pmlr-v108-khemakhem20a}.

\vspace{1.0ex}
\noindent\textbf{Example 2: Causal effects with proxy variables for an unmeasured confounder.}
Miao, Geng, and Tchetgen Tchetgen show that causal effects can be nonparametrically identified in the presence of an unmeasured
confounder $U$ if one observes multiple proxy variables with the right conditional-independence and completeness structure
\citep{MiaoGengTchetgenTchetgen2018ProxyConfounder}.
\[
\begin{aligned}
\textbf{Latent mechanism }S &: S=(U,\text{(other causes)}),\\
\textbf{Observed }(X,Y) &: \text{treatment }X,\ \text{outcome }Y,\\
\textbf{Proxies }(W,Z) &: W,Z\ \text{proxy variables influenced by }U,\\
\mathcal A_{\mathrm{mech}} &: \text{structural confounding: }U\to X,\ U\to Y,\ X\to Y,\\
\mathcal A_{\mathrm{obs}} &: \text{proxy structure yielding testable conditional independences}\\
&\text{and an operator invertibility/completeness condition},\\
r &: \text{observe }(Y,X,W,Z)\ \text{in one environment (no intervention required).}\\[0.5ex]
q(M) &: P(Y\mid do(X=x))\ \text{or an average treatment effect.}
\end{aligned}
\]
Here the identifying structure lives crucially in $\mathcal A_{\mathrm{obs}}$: observing \emph{two} appropriately related proxies and
assuming a completeness/invertibility property gives enough information to identify $q(M)$ even though $U$ is never observed
\citep{MiaoGengTchetgenTchetgen2018ProxyConfounder}. The example directly illustrates the paper’s theme: mechanism discovery (a
causal effect) is not guaranteed by flexible prediction, but can become possible when the \emph{measurement channel} is structured
and stated explicitly.

\vspace{1.0ex}
\noindent\textbf{Example 3: Causal representation learning from grouped observables.}
Morioka and Hyv\"arinen establish identifiability for causal representation learning by imposing a grouping/modularity structure on
the observation process: observational variables can be partitioned into groups that each depend on a particular latent factor
\citep{pmlr-v235-morioka24a}.
\[
\begin{aligned}
\textbf{Latent mechanism }S &: (S_1,\dots,S_M)\ \text{with a causal model among blocks},\\
\textbf{Proxies }X &: X=(X_{G_1},\dots,X_{G_M})\ \text{partitioned into observed groups }G_m,\\
\mathcal A_{\mathrm{mech}} &: \text{a causal model among }S_1,\dots,S_M\ \text{(allowing confounders/cycles under stated conditions)},\\
\mathcal A_{\mathrm{obs}} &: X_{G_m}\ \text{depends (primarily) on }S_m\\
&\text{via a group-specific mixing/channel; groups are modular},\\
r &: \text{single environment; observe all groups (the key is the grouped measurement structure).}\\[0.5ex]
q(M) &: \text{latent causal structure and/or latent features (up to allowable reparameterizations).}
\end{aligned}
\]
This is a clean example where identification is purchased mainly by \emph{measurement modularity}: the grouping assumption shrinks
$\mathcal M(P_X)$ by ruling out alternative mixings that would otherwise be observationally equivalent \citep{pmlr-v235-morioka24a}.
For a precise theorem in this direction, see Theorem~1 of \citet{pmlr-v235-morioka24a}.
Appendix~\ref{app:min_ident} provides a complementary, more general template for block identification and orientation under grouped
observables plus additional asymmetries.

\vspace{1.0ex}
\noindent\textbf{Example 4: Orientation and causal features from invariance across environments (ICP).}
Invariant causal prediction treats identification as an invariance-matching problem across environments: causal parents are those
predictors for which the conditional law of $Y$ given those predictors remains stable across environments \citep{peters2016icp}.
\[
\begin{aligned}
\textbf{Latent mechanism }S &: \text{a causal model generating }(X,Y),\\
\textbf{Observed }(X^{(e)},Y^{(e)}) &: \text{data collected across environments }e\in\mathcal E,\\
\mathcal A_{\mathrm{mech}} &: \exists S^\star\subseteq\{1,\dots,p\}\ \text{s.t. }P(Y^{(e)}\mid X^{(e)}_{S^\star})\ \text{is invariant in }e,\\
\mathcal A_{\mathrm{obs}} &: \text{environment labels are observed;}\\
&\text{shifts affect non-causal parts or inputs but preserve the structural law},\\
r &: \text{multi-environment data collection (natural experiments, domain shifts, designed perturbations).}\\[0.5ex]
q(M) &: S^\star\ \text{(a set of causal predictors/parents, or constraints implied by invariance).}
\end{aligned}
\]
This example emphasizes the role of design: identification improves because $r$ supplies variation that breaks observational
equivalence. More i.i.d.\ data from a single environment would not create the invariance contrasts that ICP exploits. This example presents
a template, not a new theorem: it synthesizes the ICP invariance principle \citep{peters2016icp} with grouped-observable
latent-structure results such as \citet{pmlr-v235-morioka24a}.

\vspace{1.0ex}
\noindent\textbf{Example 5: Graph-based identification of interventions (do-calculus and the ID algorithm).}
In causal inference, a causal graph plus intervention semantics can turn an observational distribution into an identifiable causal
quantity. Pearl’s framework formalizes interventions via $do(\cdot)$, and Shpitser and Pearl provide a complete identification
procedure (the ID algorithm) for many settings with hidden variables \citep{Pearl2009,shpitserpearl2006aaai}.
\[
\begin{aligned}
\textbf{Mechanism }M &: \text{a structural causal model over variables }V\ \text{(possibly with latent confounders)},\\
\textbf{Observed }P(V) &: \text{the observational distribution over measured nodes},\\
\mathcal A_{\mathrm{mech}} &: \text{causal Markov property w.r.t.\ a given graph }G\\
&\text{and intervention semantics (edge cutting)},\\
\mathcal A_{\mathrm{obs}} &: \text{the measured variables correspond to nodes in }G;\\
&\text{unmeasured causes are represented as latent nodes},\\
r &: \text{may be purely observational given }G,\ \text{or strengthened by collecting additional nodes / interventions}.\\[0.5ex]
q(M) &: P(Y\mid do(X=x))\ \text{or more general interventional distributions.}
\end{aligned}
\]
Here, identifying structure is explicit and audit-ready: it is exactly the graph and intervention semantics plus any assumptions
about which variables are measured. This is the econometric identification mindset in its most literal form: what is learned is not
``the mechanism'' in the abstract, but a particular interventional query that is identifiable under stated structure.

\vspace{1.0ex}
\noindent\textbf{Example 6: Equation discovery in physical systems (symbolic regression, sparse dynamics, PINNs).}
A large class of ``scientific ML'' methods succeed precisely because they impose strong structure on both the mechanism and the
measurement process \citep{Schmidt2009,Brunton2016,Raissi2019}. The details differ, but the identifying logic is the same.
\[
\begin{aligned}
\textbf{Mechanism }M &: \text{a low-complexity governing law (ODE/PDE) with constrained form},\\
\textbf{Observed }X &: \text{measurements of system states (and sometimes derivatives) under selected conditions},\\
\mathcal A_{\mathrm{mech}} &: \text{a restricted hypothesis class: sparsity, conservation,}\\
&\text{known operator library, smoothness},\\
\mathcal A_{\mathrm{obs}} &: \text{measurement access to the relevant state variables and/or their derivatives,}\\
&\text{with controlled noise/coarsening},\\
r &: \text{designed initial/boundary conditions, forcing, or experimental settings that excite informative variation}.\\[0.5ex]
q(M) &: \text{the governing equation/operator and/or its parameters (and implied invariances/derivatives).}
\end{aligned}
\]
Concretely, symbolic regression restricts the form of candidate laws and searches for simple expressions consistent with data
\citep{Schmidt2009}. Sparse identification (SINDy) assumes that $\dot x(t)=\Theta(x(t))\xi$ for a known feature/library map
$\Theta(\cdot)$ with a sparse coefficient vector $\xi$, and identification depends on measuring (or stably approximating) derivatives
and on trajectories that sufficiently ``excite'' the library \citep{Brunton2016}. Physics-informed neural networks incorporate the
constraint that a PDE residual must vanish and leverage boundary/initial conditions as additional structure; inverse problems
remain identification problems whose solvability depends on what is observed and what regimes are included \citep{Raissi2019}.
Across these methods, the moral is identical to the paper’s: mechanistic success comes from explicit, checkable identifying
structure---not from black-box flexibility alone.
\endgroup

\section{A Formal Definition of Identifying Structure}
\label{app:ident_structure_def}

This appendix makes precise the meaning of ``identifying structure'' and why it necessarily has \emph{three} components:
assumptions about the latent mechanism, assumptions about the proxy process, and assumptions about design/data collection. The formal definition presents this as two assumption bundles (one for the mechanism and one for the process), combined with a design/data collection regime. The definition is
query-specific: identification is always relative to a mechanistic question $q$ (e.g., a counterfactual effect, a derivative sign,
or an invariance claim).

\vspace{0.5ex}
\noindent\textbf{Mechanisms and observational regimes.}
Let $\mathcal M$ be a set of candidate mechanisms. An element $M\in\mathcal M$ should be understood broadly as a complete
mechanistic story, including:
(i) latent mechanism-level variables $S$ with some law $P_S$,
(ii) an observation or measurement channel $K$ generating proxies $X$ from $S$,
and (iii) any additional structural laws needed to answer mechanistic queries (e.g., an outcome law $P(Y\mid S)$, or a causal model
supporting interventions). Write this schematically as
\[
M \;=\; (P_S,\; K,\; \text{(structural laws / intervention semantics)}).
\]

Data collection enters through an \emph{observational regime} $r$. The regime specifies which environments or interventions are
observed and what variables are measured in each. Formally, let $\mathcal E$ index environments (which may include explicit
interventions), and let $O^{(e)}$ denote the observed variables in environment $e$ (often $O^{(e)}=(X^{(e)},Y^{(e)})$). The
regime induces an \emph{observation operator}
\[
T_r:\ \mathcal M \to \prod_{e\in\mathcal E} \mathcal P(\mathcal O^{(e)}),
\qquad
T_r(M)=\bigl\{P_{O}^{(e)}(M)\bigr\}_{e\in\mathcal E},
\]
mapping each mechanism $M$ to the family of observable distributions it implies under the regime. The single-environment proxy-only
setting in the main text corresponds to $\mathcal E=\{e_0\}$ and $O^{(e_0)}=X$, so $T_r(M)=P_X(M)$.

\vspace{0.5ex}
\noindent\textbf{Assumptions split into mechanism-side and observation-side.}
Let $\mathcal A$ denote a set of restrictions on mechanisms, i.e., a subset $\mathcal A\subseteq \mathcal M$. To reflect the
logic in the paper, decompose restrictions into
\[
\mathcal A \;=\; \mathcal A_{\mathrm{mech}}\ \cap\ \mathcal A_{\mathrm{obs}},
\]
where $\mathcal A_{\mathrm{mech}}$ constrains the latent mechanism and its interventional semantics (e.g., causal graph, functional
form, invariances, monotonicities, sparsity, conservation laws), and $\mathcal A_{\mathrm{obs}}$ constrains the measurement process
linking $S$ to $X$ (e.g., grouping/modularity, conditional independence, completeness/informativeness, known measurement error
structure, anchor variables, repeated measurements).

\vspace{0.5ex}
\noindent\textbf{Compatibility sets and identified sets.}
Fix an observational regime $r$ and an observed family of distributions
$\{P_{O}^{(e)}\}_{e\in\mathcal E}$ (the population object that the dataset estimates). The set of mechanisms compatible with this
evidence under assumptions $\mathcal A$ is
\[
\mathcal M\!\left(\{P_{O}^{(e)}\}_{e\in\mathcal E};\ \mathcal A,\ r\right)
\;=\;
\Bigl\{\,M\in\mathcal A:\ T_r(M)=\{P_{O}^{(e)}\}_{e\in\mathcal E}\,\Bigr\}.
\]
This generalizes the main-text compatibility set $\mathcal M(P_X)$, which corresponds to $\mathcal A=\mathcal M$ and a regime
with one environment and $O=X$.

Let $q:\mathcal M\to\mathcal Q$ be a mechanistic query (e.g., a causal effect, an intervention response curve, a derivative sign,
an invariance statement). The \emph{identified set} of answers under $(\mathcal A,r)$ is
\[
\mathcal I_q\!\left(\{P_{O}^{(e)}\}_{e\in\mathcal E};\ \mathcal A,\ r\right)
\;:=\;
\Bigl\{\,q(M):\ M\in \mathcal M(\{P_{O}^{(e)}\};\mathcal A,r)\,\Bigr\}.
\]
The query $q$ is \emph{point-identified} under $(\mathcal A,r)$ if $\mathcal I_q(\cdot)$ is a singleton; otherwise $q$ is only
\emph{partially identified}, and the width/geometry of $\mathcal I_q(\cdot)$ is the population-level object corresponding to the
\emph{structural identification uncertainty} emphasized in the paper.

\vspace{0.5ex}
\noindent\textbf{Definition (Identifying structure).}
Fix a query $q$ and an observational regime $r$. An \emph{identifying structure} for $q$ (at the population level) is any pair
\[
(\mathcal A_{\mathrm{mech}},\ \mathcal A_{\mathrm{obs}})\quad\text{together with}\quad r
\]
such that the identified set $\mathcal I_q(\{P_O^{(e)}\};\mathcal A,r)$ is a singleton (or, in applications, is sufficiently small
to support the intended scientific claim). The key point is that identifying structure is inherently joint: mechanism-side
restrictions alone do not identify $q$ if measurement can be reparameterized to preserve the same proxies, and observation-side
restrictions alone do not identify $q$ if the mechanism class remains too broad. Discriminating data collection is represented by
$r$: changing environments, adding interventions, or adding measurement channels can change the family $\{P_O^{(e)}\}$ and shrink
the compatibility set, even when ``more data'' within a fixed environment cannot.

Appendix~\ref{app:min_ident} instantiates this definition for a particular target: identifying a modular \emph{block} structure of
latent variables (up to within-block reparameterizations) and orienting relations among blocks. The point of the definition above
is broader: it cleanly separates (i) what is assumed about mechanisms, (ii) what is assumed about measurement, and (iii) what is
purchased by design.

\section{Minimal Structure for Identifiability (Blocks and Orientation)}
\label{app:min_ident}
Sections~\ref{sec:proxy_gap}--\ref{sec:narrative} emphasize non-identification: from proxies alone, many mechanistically distinct
stories can be observationally indistinguishable. This appendix records a complementary point that is essential for a non-nihilist
position: \emph{identifiability becomes possible} once one states and defends additional structure. The goal here is not to develop
a full theory, but to give a minimal, checkable template for when (i) a modular \emph{block} structure of latent drivers can be
recovered, and (ii) \emph{directions} among blocks can be oriented.
The results in this appendix are stated as templates: they summarize the kinds of additional assumptions under which identifiability can be obtained, rather than presenting new identifiability theorems proved in this paper.

\vspace{0.5ex}
\noindent\textbf{Setup.}
Let $S=(S_1,\dots,S_B)$ denote latent mechanism variables (``blocks''), and let $X=(X_1,\dots,X_p)$ denote observed proxies.
Assume proxies are generated from $S$ by either a deterministic measurement map $X=f(S)$ or a stochastic measurement channel
$X\mid S \sim K(\cdot\mid S)$. As in the main text, identification is understood \emph{up to within-block reparameterizations}:
for invertible $g_b$, replacing $S_b$ by $g_b(S_b)$ and pulling back the measurement channel leaves $P_X$ unchanged.

\subsection*{A. Block structure from grouped observables}
The strongest form of ``minimal structure'' is a \emph{grouping} assumption: each subset of observed variables is driven primarily
by one latent block, and cross-block relations show up as dependence \emph{between} groups.

\vspace{0.5ex}
\noindent\textbf{Assumption A1 (Grouping / modular measurement).}
There exists a partition of proxy coordinates into $B$ nonempty groups,
$\{1,\dots,p\}=\bigsqcup_{b=1}^B G_b$, such that for each $b$ the group $X_{G_b}$ depends on the latent block $S_b$ through a
group-specific map or channel:
\[
X_{G_b} \mid S \;\sim\; K_b(\,\cdot \mid S_b\,),
\qquad\text{and}\qquad
K(\cdot\mid S)=\prod_{b=1}^B K_b(\cdot\mid S_b).
\]
(Deterministic measurement is included by taking $K_b$ to be a point mass at $f_b(S_b)$.)

\vspace{0.5ex}
\noindent\textbf{Assumption A2 (Channel informativeness).}
Each group-channel $K_b$ is informative enough that distinct functions of $S_b$ cannot be ``hidden'' completely by measurement.
A standard sufficient condition is an injectivity/completeness-type property: if $\varphi(S_b)$ is nontrivial, then its footprint
is detectable from $X_{G_b}$ (e.g., the conditional expectation operator
$\varphi \mapsto \mathbb{E}[\varphi(S_b)\mid X_{G_b}=\cdot]$ is injective on a rich class).

\vspace{0.5ex}
\noindent\textbf{Assumption A3 (Non-separable cross-block dependence).}
The joint law of $S$ is not additively separable across the true blocks. One convenient way to state this is graphically:
there exists an undirected graph $H$ on $\{1,\dots,B\}$ such that for each edge $\{b,b'\}$, $S_b$ and $S_{b'}$ exhibit genuine
cross-block dependence (not removable by within-block reparameterizations), and $H$ is connected.

\vspace{0.5ex}
\noindent\textbf{Template A.1 (Block identifiability up to within-block reparameterization).}
Under Assumptions A1--A3, the grouping of proxies into latent blocks is identifiable from $P_X$ up to permutation of the block labels
and within-block reparameterizations. In particular, any alternative representation that preserves $P_X$ must correspond to the same
coarsest partition of proxy coordinates into conditionally independent groups given the latent blocks, modulo relabeling.

\vspace{0.5ex}
\noindent\emph{Discussion.}
This template is intentionally stated at a high level: it isolates what must be made explicit to move from ``many mechanisms fit the proxies''
to ``the proxies pick out a modular decomposition.'' The critical ingredients are (i) a defensible measurement grouping, (ii) informative channels
so that group-level signals are not annihilated by coarsening/noise, and (iii) cross-block dependence strong enough that blocks cannot be merged or
split without changing $P_X$. Recent causal representation learning results formalize identifiability from such grouped-observable structure and show
consistent estimation under concrete generative conditions.

\subsection*{B. Orientation: two minimal routes}
Once blocks are identified up to reparameterization, orienting relations among blocks requires \emph{additional asymmetry}. Two common sources are
(i) \emph{multiple environments} (invariance under interventions/shifts) and (ii) \emph{signed asymmetries} in the joint law.

\vspace{0.5ex}
\noindent\textbf{Route 1: Orientation from invariances across environments.}
Suppose the data come from environments $e\in\mathcal{E}$ with proxy distributions $P_X^{(e)}$ induced by latent laws $P_S^{(e)}$ and the
\emph{same} measurement channels $K_m$. Assume the causal mechanism among blocks is invariant while some environments perturb the distribution of
certain blocks (e.g., shift interventions).

\vspace{0.5ex}
\noindent\textbf{Assumption B1 (Invariant mechanism, shifting inputs).}
There exists a directed acyclic graph (DAG) $G$ over blocks $S_1,\dots,S_B$ such that for each environment $e$ the conditional law of each block
given its parents in $G$ is invariant in $e$, while the marginal law of some blocks varies with $e$.

\vspace{0.5ex}
\noindent\textbf{Template A.2 (Orientation from invariance).}
Under Assumptions A1--A3 and B1 (plus mild regularity so that invariances are testable at the level of $X$ via the informative channels),
the parent sets in $G$ are identifiable (up to the same within-block reparameterizations), hence the edges of $G$ can be oriented by selecting the
unique DAG whose implied invariance statements match the observed family $\{P_X^{(e)}\}_{e\in\mathcal{E}}$.

\vspace{0.5ex}
\noindent\emph{Discussion.}
This is the canonical ``scientific'' route: orientation is purchased by \emph{variation across environments} that leaves the mechanism invariant.
In practice it corresponds to: collect (or exploit) multi-environment data, then test which conditional relations remain stable across environments.

\vspace{0.5ex}
\noindent\textbf{Route 2: Orientation from a sign-faithful observable asymmetry.}
Alternatively, direction can be identified when the observed law contains an asymmetry whose sign is preserved by the admissible
measurement transformations and tracks the arrow direction on an already identified skeleton.

\vspace{0.5ex}
\noindent\textbf{Assumption B2' (Observable sign-faithful orientation score).}
For each adjacent unordered pair of identified blocks $\{b,b'\}$ in the skeleton $H$, there exists an observable measurable
functional
\[
\Delta_{bb'}=\Delta_{bb'}(X_{G_b},X_{G_{b'}})
\]
such that:
(i) $\Delta_{bb'}=-\Delta_{b'b}$ a.s.;
(ii) $\Delta_{bb'}$ is invariant under admissible within-block reparameterizations and preserved by the measurement model; and
(iii) whenever $b$ and $b'$ are adjacent,
\[
\mathbb E[\Delta_{bb'}] > 0 \iff b\to b',
\qquad
\mathbb E[\Delta_{bb'}] < 0 \iff b'\to b,
\]
and in particular $\mathbb E[\Delta_{bb'}]\neq 0$.

\vspace{0.5ex}
\begin{proposition}[Orientation from a sign-faithful observable asymmetry]\label{prop:orientation_asymmetry}
Assume A1--A3, and suppose the undirected skeleton $H$ over blocks has been identified. Under Assumption B2', each adjacent edge
$\{b,b'\}$ in $H$ is oriented by the sign of $\mathbb E[\Delta_{bb'}]$: it is oriented as $b\to b'$ if
$\mathbb E[\Delta_{bb'}]>0$, and as $b'\to b$ if $\mathbb E[\Delta_{bb'}]<0$.
\end{proposition}
\noindent\textbf{Proof.}
Fix an adjacent edge $\{b,b'\}$ in the identified skeleton $H$. By Assumption B2', the functional $\Delta_{bb'}$ is observable and
invariant under admissible within-block reparameterizations, so the sign of $\mathbb E[\Delta_{bb'}]$ is determined by the observed
law of $(X_{G_b},X_{G_{b'}})$ and does not depend on the choice of latent coordinates within each block. The sign-faithfulness clause
implies
\[
\mathbb E[\Delta_{bb'}] > 0 \iff b\to b',
\qquad
\mathbb E[\Delta_{bb'}] < 0 \iff b'\to b.
\]
Because $\mathbb E[\Delta_{bb'}]\neq 0$, exactly one of these cases holds, so the sign of $\mathbb E[\Delta_{bb'}]$ orients the
edge uniquely. The argument is edgewise, so the result is local to adjacent edges in $H$.
\hfill\(\square\)

\vspace{0.5ex}
These templates are not claims that identifiability should always be expected; rather, they illustrate that claims of mechanism discovery must be paired with an
explicit statement of \emph{which} identifying structure is being assumed (grouping, invariances across environments, anchors/asymmetries, etc.)
and \emph{how} it is being tested. Absent such structure, Sections~\ref{sec:proxy_gap}--\ref{sec:narrative} apply: proxy-rich, high-dimensional settings generically leave large mechanism equivalence classes, and fluent single-narrative explanations are epistemically unsafe.

\end{document}